\documentclass{article}
\usepackage{graphicx} 

\usepackage[T1]{fontenc}
\usepackage[utf8]{inputenc}
\usepackage{lmodern}
\usepackage{amsmath,amssymb,amsthm,mathtools}
\usepackage{bm}
\usepackage{enumitem}
\usepackage{hyperref}
\usepackage{natbib}
\usepackage{pifont}
\usepackage{booktabs}
\usepackage{tabularx}
\usepackage{float}
\usepackage{xcolor}
\usepackage{fontawesome5}
\usepackage{tikz}
\newcommand{\paperbadge}[1]{%
  \tikz[baseline=(badge.base)]\node[draw=black!60, fill=black!4, rounded corners=3pt,
    inner xsep=6pt, inner ysep=3pt, font=\small\sffamily] (badge) {#1};%
}
\setcitestyle{authoryear} 
\usepackage[final,main]{neurips_2026}
\definecolor{revisionred}{RGB}{180,35,45}
\definecolor{cameraready}{RGB}{0,0,0}
\newcommand{\rev}[1]{\textcolor{cameraready}{#1}}
\newenvironment{revision}{\par\begingroup\color{cameraready}}{\par\endgroup}
\newenvironment{revisiondiff}{\par\begingroup}{\par\endgroup}

\newtheorem{theorem}{Theorem}

\newtheorem*{remark}{Remark}

\title{Extracting Governing Equations from Latent Dynamics via Multi-View Contrastive Learning}
\author{%
    Paolo Muratore\\
    EPFL\\
    \texttt{paolo.muratore@epfl.ch} \\
    \And
    Mackenzie Weygandt Mathis\\
    EPFL\\
    \texttt{mackenzie.mathis@epfl.ch}
}

\begin{document}

\maketitle

\begin{abstract}
Identifying latent dynamical systems from noisy, high-dimensional measurements is a central problem at the intersection of representation learning, system identification, and scientific discovery. We present \textsc{dysco} \rev{(DYnamical Systems via Contrastive Optimisation)}, a multi-view temporal contrastive learning algorithm that jointly recovers latent trajectories and the governing dynamics from such observations, by leveraging multiple independent noisy views of the same underlying process to disentangle signal from noise. By parameterizing the dynamics in a structured functional basis, our framework further enables symbolic recovery of the governing equations within an affine gauge. We offer theoretical guarantees for strong identification up to an affine indeterminacy, extending prior identifiability results to the realistic setting of noisy nonlinear observations. Empirically, we demonstrate accurate recovery of both latent trajectories and flow fields across a diverse set of dynamical regimes (\emph{e.g}, chaotic, oscillatory, and metastable) under both Gaussian and Poisson observation noise, the latter being particularly relevant for neural \rev{data, for which we offer a preliminary demonstration by recovering a known rotational structure in motor-cortical recordings.}
\end{abstract}

\makeatletter
\if@anonymous\else
\vspace{-1.75ex}
\centerline{%
  \hypersetup{hidelinks}%
  \href{https://github.com/AdaptiveMotorControlLab/dysco}{\paperbadge{\raisebox{-0.12ex}{\faCode}\, Code}}%
  \hspace{0.7em}%
  \href{https://adaptivemotorcontrollab.github.io/dysco/}{\paperbadge{\raisebox{-0.12ex}{\faGlobeAmericas}\, Website}}}
\vspace{-2.75ex}
\fi
\makeatother

\section{Introduction}
A central challenge in neuroscience and machine learning is recovering the computational laws governing a system from noisy, high-dimensional observations of its activity. This problem is classically articulated in Marr’s tri-level framework \citep{marr1982vision}, which distinguishes between: the computational level, \emph{e.g.}, sorting numbers in ascending order, the algorithmic level, \emph{e.g.}, a particular procedure, such as bubble-sort or merge-sort, and the implementation level, \emph{i.e.}, the circuit of transistors switching on and off. While modern data-driven methods have made substantial progress at distilling the implementation level of raw neural activity into abstract latent representations \citep{schneiderLearnableLatentEmbeddings2023,mathis2026joint}, the principled recovery of the dynamical laws governing these representations (\emph{i.e.} the complete algorithm description of the system) remains an open challenge \citep{sussilloOpeningBlackBox2013, vyasComputationNeuralPopulation2020, versteegComputationthroughDynamicsBenchmarkSimulated2025}. Bridging this gap is essential for interpretability, scientific discovery and the development of robust, generalizable models.

\begin{figure}[tb]
    \centering
    \includegraphics[width=\linewidth]{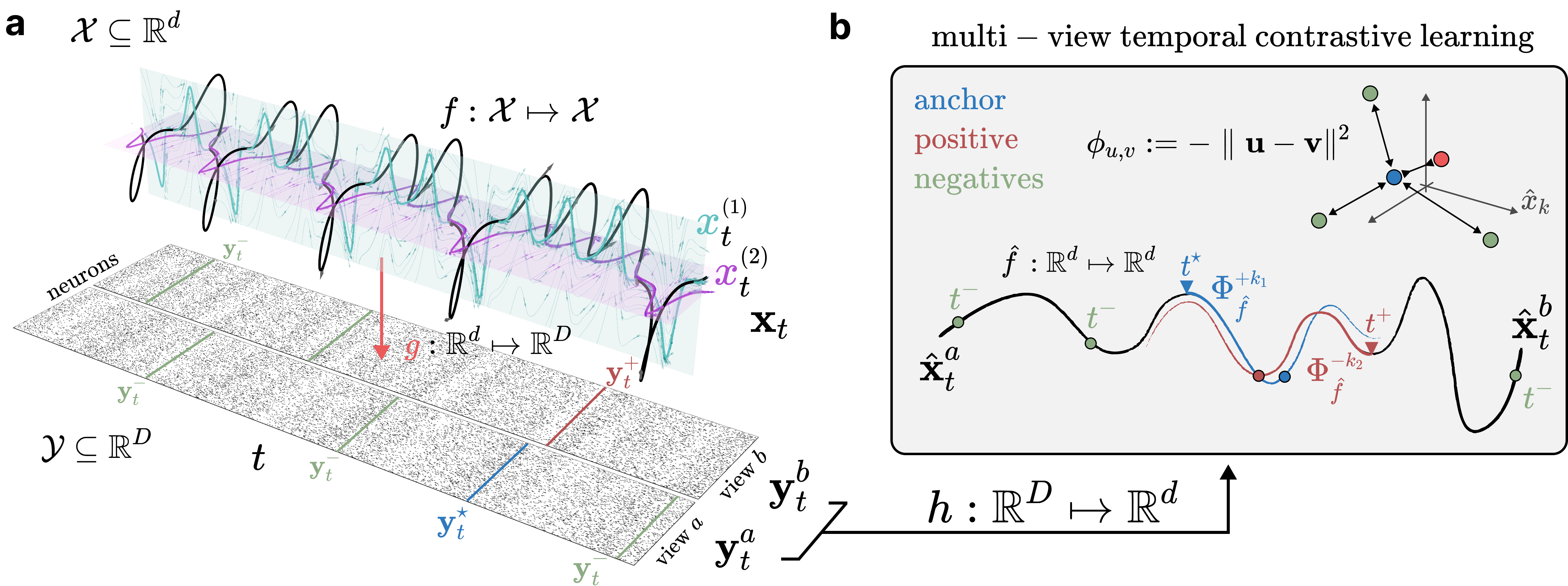}
    \caption{\textbf{Graphical overview.} (\textbf{a}) We consider a latent dynamical system (top) that evolves according to an unknown \rev{transition} $f : \mathcal{X} \mapsto \mathcal{X}$. We observe the system via a non-linear mixing-channel $g: \mathbb{R}^d \mapsto \mathbb{R}^D$ that produces noisy high-dimensional observations $\bm{y}_t^a$. (\textbf{b}) We solve the identifiability problem by simultaneously learning an encoder $h: \mathbb{R}^D \mapsto \mathbb{R}^d$ and a symbolic \rev{vector field} $\hat f$ by minimizing the multi-view temporal contrastive loss defined in \eqref{eq:minimisation_problem}.}
    \label{fig:model_overview}
\end{figure}

In many domains, ranging from neural population recordings to complex physical systems, the observable data are high-dimensional, noisy and generated through nonlinear measurement processes. The underlying system, however, is often governed by a low-dimensional latent dynamical process \citep{mastrogiuseppe2018linking}. Recovering both the latent state and the governing dynamics from such observations constitutes a fundamental inverse problem at the intersection of representation learning, system identification, and scientific modeling. A key obstacle is that noisy, nonlinear observations sever the identifiability guarantees that underpin existing approaches: noise corrupts the signal on which both representation learning and symbolic methods rely, and without multiple independent views of the same latent state there is no principled mechanism to separate signal from nuisance. Recent advances in self-supervised and contrastive learning have demonstrated that temporal structure provides a powerful inductive bias for extracting meaningful representations from sequential data. However, existing approaches primarily focus on learning useful representations, rather than explicitly identifying the underlying dynamical system that generates the data. 
In parallel, methods for symbolic regression and dynamical system discovery aim to recover explicit functional forms of the dynamics \citep{bruntonDiscoveringGoverningEquations2016}, enabling interpretability and mechanistic insight. Yet, these approaches typically assume direct or partially observed access to the system state (but see \citep{championDatadrivenDiscoveryCoordinates2019} for a autoencoder approach to latent dynamical systems), limiting their applicability in realistic settings where observations are indirect and noisy. This disconnect highlights a key open problem: \emph{can we simultaneously learn latent representations and recover the governing dynamical laws from high-dimensional, noisy observations, with theoretical guarantees of identifiability?}

In this work, we address this problem by introducing a multi-view temporal contrastive learning framework for the identification of latent dynamical systems.
Our approach leverages multiple independent noisy views of the same underlying system to enforce consistency constraints that effectively disentangle signal from noise, while a structured parameterization of the dynamics enables symbolic recovery of the governing equations. Concretely, we consider a setting where a latent dynamical system evolves according to an unknown nonlinear map and is observed through multiple noisy, nonlinear measurement channels. We jointly learn (\emph{i}) an encoder that inverts the observation process and recovers latent states, and (\emph{ii}) a structured dynamics model expressed in a predefined functional basis. The learning signal is provided by a temporal contrastive objective that exploits both time evolution and cross-view consistency. As illustrated in our formulation (see Figure \ref{fig:model_overview}), this setup allows the model to align representations across views while respecting the underlying dynamics. Our main contributions are as follows:

\begin{itemize}
\item [\emph{(i)}]\emph{Multi-view contrastive system identification}. We propose \textsc{dysco}, a \rev{multi-view extension of dynamical} contrastive learning that leverages temporal structure and multiple noisy views to jointly recover latent states and dynamics from high-dimensional observations.

\item [\emph{(ii)}]\emph{Theoretical identifiability guarantees}. We show that, under asymptotic assumptions, multi-view contrastive learning identifies the latent dynamical system up to an affine transformation, extending prior results to the realistic setting of noisy, nonlinear observations.

\item [\emph{(iii)}]\emph{Compatibility with symbolic regression}. We show how the affine gauge structure of the recovered representation is compatible with downstream symbolic-regression pipelines, and illustrate this on the Lorenz system.

\item [\emph{(iv)}]\emph{Empirical validation across dynamical regimes}. We demonstrate accurate recovery of both latent trajectories and flow fields across a diverse set of systems, including chaotic, oscillatory, and metastable dynamics, even under substantial observational noise.
\end{itemize}

Taken together, our results suggest that contrastive learning, when appropriately structured, provides a principled route toward solving Marr’s inverse problem: moving from raw observations at the implementation level to explicit, interpretable descriptions of the underlying computation.

\textbf{Related Work} Recovering latent dynamical systems from high-dimensional observations sits at the intersection of representation learning, system identification, and symbolic regression. The work most closely related to ours is \citep{laiz2025}, which demonstrates that contrastive objectives can recover both latent representations and nonlinear dynamics up to an affine transformation. Our work builds on this foundation but departs from it in two important ways: first, we explicitly leverage multi-view observations to handle realistic noisy measurement processes, where \citep{laiz2025}'s identifiability relies on noiseless observations---a condition that multi-view consistency is specifically designed to relax by forcing the encoder to extract only the shared, noise-free component across views; second, we enforce a symbolic structure on the learned dynamics, moving beyond the switching-linear dynamical system formulation and enabling recovery of interpretable governing equations. We situate this contribution below within the broader literature.

Recent advances in self-supervised learning, and in particular contrastive learning, have demonstrated that temporal structure can provide a powerful signal for extracting latent representations. Methods such as Contrastive Predictive Coding \citep{oord2018representation} and its variants \citep{schneider2019wav2vec, henaff2020data} leverage predictive objectives to learn embeddings that capture underlying dynamics. However, these approaches are typically optimized for representation quality rather than explicit recovery of the governing dynamical system. This gap has motivated a line of theoretical work connecting contrastive objectives to identifiable representation learning. In particular, nonlinear ICA frameworks \citep{hyvarinenUnsupervisedFeatureExtraction2016, hyvarinenNonlinearICAUsing2019, halvaDisentanglingIdentifiableFeatures2021} show that temporal or auxiliary structure can render latent variables identifiable under suitable conditions (see \citep{hyvarinenNonlinearIndependentComponent2023} for a review), and recent results have further linked contrastive learning to inversion of the data-generating process \citep{zimmermannContrastiveLearningInverts2021}. Nevertheless, these works focus primarily on recovering latent states, leaving the identification of the underlying dynamics largely unaddressed. \rev{More recently, VAE-based techniques such as iLQR-VAE \citep{schimel2022ilqr} and \citet{brenner2024multimodal} leveraged control-based variational inference and multimodal signals for dynamical reconstruction. DYSCO on the other hand is encoder-only and uses the correspondence between noisy measurements of an aligned latent state to recover the dynamics in a symbolic form.}

In neuroscience, related efforts have focused on inferring latent population dynamics from noisy neural recordings. Sequential autoencoder models such as LFADS \citep{pandarinathInferringSingletrialNeural2018a, kimInferringLatentDynamics2021, shahImprovedInterpretabilityLFADS2025} and more recent contrastive approaches such as CEBRA \citep{schneiderLearnableLatentEmbeddings2023} learn low-dimensional representations that capture behaviorally relevant structure. While highly effective empirically, these methods typically either lack identifiability guarantees or do not recover explicit dynamical laws.

The symbolic identification of dynamical systems has been extensively studied in the context of sparse regression methods such as SINDy \citep{bruntonDiscoveringGoverningEquations2016, klishinStatisticalMechanicsDynamical2025}, which recovers governing equations by selecting a sparse combination of candidate basis functions. More recent approaches extend this idea using sequence \citep{dascoliODEFormerSymbolicRegression2023} or diffusion-based models \citep{bastiani2025diffusion} for symbolic regression. These methods, however, typically assume direct (or partial) access to the system state. Closer in spirit, \cite{championDatadrivenDiscoveryCoordinates2019} adapted SINDy within a variational autoencoder framework for discovery of latent dynamics. However this framework lacks any theoretical guarantee and the multi-term loss function requires careful balancing of the relevant terms. In contrast, our approach provides theoretical identifiability guarantees within an affine gauge using a single contrastive objective, with symbolic recovery as a downstream step.

\section{Contrastive Learning Identifies Latent Dynamical Systems}
\label{sec:contrastive_learning_identifies}

\textbf{Problem definition} We study the following non-linear identification problem. We consider a latent (\emph{i.e.}, non-observable) state variable $\bm{x}_t \in \mathcal{X} \subseteq \mathbb{R}^d$ whose stochastic discrete-time evolution can be described via a dynamical system of the form:
\begin{equation}
    \bm{x}_{t+1} = \rev{F} (\bm{x}_{t} ) + C \bm{u}_t + \bm \varepsilon_t,
\label{eq:latent_dyn_sys}
\end{equation}
where \(\rev{F} : \mathcal{X} \mapsto \mathcal{X}\) is \rev{a bijective transition map}, $\bm{u}_t \in \mathbb{R}^u$ is an external forcing input, $C \in \mathbb{R}^{d \times u}$ is the input coupling matrix and $\bm \varepsilon_t \in \mathbb{R}^d$ is the latent system noise. We observe this system via the non-linear observation channel:
\begin{equation}
    \bm{y}_t^a = g(\bm{x}_t) + \bm{\xi}_t^a,
\label{eq:observation_model}
\end{equation}
where $\bm{y}_t^a \in \mathbb{R}^D$, $D \ge d$ is the observed state, $g: \mathcal{X} \to \mathbb{R}^D$ is a non-linear injective mixing function from latents to observations and $\bm \xi_t^a$ denotes the observation noise assumed independent from the latent state, with $a$ indexing independent realizations of the noise process. Our objective is the following: given a set of observed trajectories $\Upsilon \coloneq \left\{ \bm{y}_t^a \right\}_{t, a}^{T, V}$, with $T$ the total trajectory length and $V \ge 2$ the total number of independent \emph{views} of the system, reconstruct the latent states $\left\{ \bm{x}_t \right\}_t^T$ and the dynamics $\rev{F}$. In practice, we seek to learn a de-mixing function $h : \mathbb{R}^D \mapsto \mathcal{X}$ and a dynamics model $\rev{\widehat F}$ such that the composition $r \coloneqq h \circ g$ is a trivial function (ideally, the identity) and $\rev{\widehat F}$ recapitulates $\rev{F}$. \rev{For the continuous-time implementation, we parameterize the transition $F$ as a vector field $f$ and consider the solution operator $\Phi_f^{t+k\Delta,t}(\bm x;\bm u_{[t,t+k\Delta]})$. In the autonomous deterministic case, $\rev{F=\Phi_f^{\Delta,0}}$ is the sampled transition}.

\textbf{Model definition} We follow a similar construction from \citep{laiz2025} and employ a contrastive learning framework for solving the identification problem. We introduce an encoder $h_\theta : \mathbb{R}^D \mapsto \mathbb{R}^d$, \emph{i.e.}, a neural network parametrized by $\theta$ and a symbolic model for the dynamics $\hat f_\Xi : \mathbb{R}^d \mapsto \mathbb{R}^d$, represented as the span of a pre-defined set of library scalar functions $\Xi \coloneq \left\{ f_1, f_2, \dots, f_n \right\} \in \mathbb{R}^n$ \citep{bruntonDiscoveringGoverningEquations2016}. Typically in this work, we choose $\Xi$ as the set of all monomials up to a given degree, \emph{e.g.} for a two-dimensional system $\Xi^{(2)} = \left\{ 1, x, y, x^2, y^2, xy\right\}$ would represent the polynomial basis of degree two. The choice of the symbolic basis $\Xi$ can be used to trade complexity (more terms, different function primitives $f_{i}$) for expressivity. The flow-field at a given point $\hat f_\Xi (\bm x_t )$ is then expressed as:
\begin{equation}
    \hat f_\Xi ( \bm x_t) = \bm \Theta \Xi (\bm x_t) = \sum_k^n \bm \theta_k f_k (\bm x_t),
\label{eq:symbolic_flow_field}
\end{equation}
where $\bm \Theta \in \mathbb{R}^{d \times n}$ is a learnable matrix of basis coefficients and $\Xi \left( \bm x_t \right) \in \mathbb{R}^{n}$ is the vector of scalar functions evaluated at $\bm x_t$. Given \rev{a vector field} and initial state, we compute the \rev{solution using RK4. For compactness, we write the sampled solution operator as:}
\begin{equation}
    \rev{\Phi_f^k(\bm x_t;\bm u)} := \rev{\Phi_f^{t+k\Delta,t}}(\bm x_t;\rev{\bm u_{[t,t+k\Delta]}}),
    \qquad \rev{\Phi_f^0}(\bm x_t)=\bm x_t.
    \label{eq:forward_integration_operator}
\end{equation}
\rev{For the discrete-time model, the corresponding solution rollout operator is the iterate $F^k$.}


Finally, since the core objective of contrastive learning is to model the pairwise similarities of the data, we consider a similarity function $\phi: \mathbb{R}^d \times \mathbb{R}^d \mapsto \mathbb{R}$, which we will usually take to be the negative squared euclidean distance between two points. We can now define our model \textsc{DYSCO} as the following composition:
\begin{equation}
    \psi_{ab} (\bm y^a_t, \bm y^b_\tau) \coloneq \phi \left( \Phi_{\hat f}^{\pm k_1} \left( h(\bm y^a_t ) \right), \Phi_{\hat f}^{\mp k_2} \left(h (\bm y^b_\tau )\right)\right) - \alpha \left( h (\bm y^b_\tau ) \right),
\label{eq:model_definition}
\end{equation}
where $k_{1, 2} \in \left\{ 0, 1, \dots, \kappa \right\}$ with $\kappa$ a model hyperparameter defining the maximal time-integration horizon, \rev{$a, b \in V$ are view indices and $\alpha : \mathbb{R}^d \mapsto \mathbb{R}$ is a scalar potential, whose role is to supply the candidate-marginal correction in the optimal contrastive score $\log p(\bm z'\mid\bm z)-\log q(\bm z')$, up to an anchor-only term. Additionally, we have introduced the shorthand notation for backward rollouts:
\begin{equation}
    \Phi_f^{-k} \coloneq (\Phi_f^k)^{-1},
    \label{eq:backward_map}
\end{equation}
which provide optional additional training signals (our main theorem only uses an active forward one-step constituent, but we found backward rollouts to stabilize training in practise).}

Our model can then be viewed as implementing an amortized inference of the full multi-view model defined equivalently as:
\begin{equation}
    \Psi_V \left(\bm y^{(V)}_t, \bm y^{(V)}_\tau \right) \coloneq \phi \left( \Phi_{\hat f}^{\pm k_1} \left( H_V(\bm y^{(V)}_t ) \right), \Phi_{\hat f}^{\mp k_2} \left( H_V (\bm y^{(V)}_\tau )\right)\right) - \alpha' \left( H_V (\bm y^{(V)}_\tau )\right),
\label{eq:model_definition_full_views}
\end{equation}
where $H_V : \mathbb{R}^{D\times V} \mapsto \mathbb{R}^d$ is a full-multi-view encoder, $\bm y_t^{(V)} \coloneq \left\{ \bm y_t^1, \dots, y_t^V \right\}$ is the full set of observations at time $t$ and $\alpha' : \mathbb R^{d}\mapsto \mathbb R$ is the corresponding scalar potential.

From here, our learning framework follows the standard contrastive learning setup with InfoNCE objective function \citep{oord2018representation}, \emph{i.e.} we learn to minimize the following expression for the log-likelihood:
\begin{equation}
    \log p_{\psi} \left( \bm y_t^{a, \star} | \bm y_t^{b, +}, \left\{ \bm  y_t^{c, -}\right\} \right) = \psi_{ab} \left( \bm y_t^{a, \star}, \bm y_t^{b, +} \right) - \log \sum_{\bm y_t^{\ast} \in \bm y_t^{b, +} \cup \left\{ \bm y_t^{c, -} \right\}} \exp \left( \psi_{a\ast} \left( \bm y_t^{a, \star}, \bm y_t^\ast \right) \right),
\label{eq:model_log_likelihood}
\end{equation}
where we have used the following notation $\bm y_t^\star$, $\bm y_t^+$, $\left\{ \bm y_t^- \right\}$ to denote, respectively, the reference (or anchor) and positive samples and the set of negative samples. The learning problem can then be fully specified as the minimization:
\begin{equation}
    \min_\psi \mathcal{L} \left[ \psi \right] = \min_{\psi} \mathbb{E}_{a, b, c} \mathbb{E}_{k_1, k_2} \mathbb{E}_{t, \tau} \left[ -\log p_{\psi} \left( \bm y_{t\mp k_1}^{a} | \bm y_{t \pm k_2}^b, \left\{ \bm y_\tau^{c} \right\} \right) \right],
\label{eq:minimisation_problem}
\end{equation}
\begin{revisiondiff}
where all the expectations are computed by uniform sampling in the relevant ranges. Intuitively, the learning process is constructed in such a way as to exploit the \rev{temporal structure} and \rev{available forward/backward solution operators} of the latent dynamical system, while the enforced across-view consistency provides the encoder with the learning signal to effectively denoise the observations $\bm y_t^\alpha$ by forcing it to rely only on the shared component of the signal (the $g (\bm x_t )$ in \eqref{eq:observation_model}) and discard the nuisance (noise) component $\bm \xi^\alpha_t$. We offer a graphical overview of the problem in Figure \ref{fig:model_overview}.
\end{revisiondiff}

We can now state our main theoretical result for the full-views model \eqref{eq:model_definition_full_views}, formulated for the asymptotic limit $V \to \infty$ and the particular case of no external control $\bm u_t \equiv 0$. Our result extends \cite{laiz2025} to the case of noisy observations. We then show that strong identification holds well in practice also for the DYSCO model \eqref{eq:model_definition} and the case of (known) external forcing.

\begin{revisiondiff}
\begin{theorem}[Multi-view contrastive estimation of noisy dynamics]\label{thm:identifiability_theorem}
Consider the latent dynamical system $\bm x_{t+1} = \rev{F}(\bm x_t) + \bm \varepsilon_t$ with $\rev{F} : \mathbb R^d \mapsto \mathbb R^d$, a $C^2$ diffeomorphism and noise $\bm \varepsilon_t \sim \mathcal{N} \left(0, \Sigma \right)$, $\Sigma \succ 0$, independent of time. We observe the system via the non-linear observation channel \eqref{eq:observation_model} with $g : \mathbb{R}^d \mapsto \mathbb{R}^D$, $D \ge d$ a $C^1$ bi-Lipschitz embedding and sub-Gaussian noise $\bm \xi_t^a$ independent of $\bm x_t$ with $a = 1, \dots, V$. Finally, consider the full multi-view model $\Psi_V$ defined in \eqref{eq:model_definition_full_views}. Assume the following:

\begin{enumerate}
    \item[\textnormal{(A1)}] The multi-horizon objective $\mathcal L [ \Psi_V]$ is jointly realizable, \emph{i.e.} there exists a single tuple $( H_V, \rev{\widehat F}, \alpha )$ attaining the infimum of every active constituent $\mathcal L_V^{k_1,k_2}$ and $\lambda^{(1, 0)} > 0$.

    \item[\textnormal{(A2)}] In the infinite-view population limit, the population global minimizers have a stable $C^1$ limit for both representation and dynamics:
    \begin{equation*}
        H_V (\bm y_t^{(V)}) \to r(\bm x_t) \qquad \mathrm{in\ probability}
    \end{equation*}
    for $r \in C^1$ and $\rev{\widehat F_V} \to \rev{\widehat F} \in C^1$ on the learned support. \rev{The potential class represents the marginal correction. Density and observation regularity conditions are given in Appendix~\ref{sec:appendix_theorem_proofs}.}
\end{enumerate}

Then, in the limit $V, T \to \infty$, any global minimizer identifies the latent state and the dynamics up to a common affine indeterminacy, \emph{i.e.}  there exist $L\in\mathrm{GL}(d)$ and $\bm b\in\mathbb R^d$ such that for every point $\bm x \in \mathcal{U}$:
\begin{equation}
    r (\bm x) = L \bm x + \bm b,
\end{equation}

moreover, the learned deterministic dynamics $\rev{\widehat F}$ are the affine conjugate of the true dynamics:
\begin{equation}
    \rev{\widehat F} (\bm z) = L \rev{F} \left( L^{-1}\rev{(\bm z} - \bm b) \right) + \bm b , \qquad \bm z \in r(\mathcal U).
\end{equation}

\begin{proof}
    See Appendix \ref{sec:appendix_theorem_proofs} for the full proof and subsequent remarks.
\end{proof}
    
\end{theorem}
\end{revisiondiff}

\section{Experimental Setup}
\label{sec:experimental_setup}

We validate our theoretical results by verifying that our algorithm can recover known latent dynamical systems. We explored a wide variety of dynamical regimes and external forcing configurations, ranging from chaotic (\emph{e.g.} Lorenz attractor) to metastable (\emph{e.g.} heteroclinic) to periodic (\emph{e.g.} Duffing oscillator) with fast-and-slow dynamics (\emph{e.g.} FitzHugh-Nagumo oscillator), proving that our system is robust and applicable to diverse, realistic scenarios relevant to machine learning and neuroscience.

\textbf{Data Generating Process} Given a dynamical system of the form \eqref{eq:latent_dyn_sys} (or the equivalent continuous-time formulation specified via the corresponding ODE) we generated a collection of latent trajectories $\left\{ \bm x_t^{i}\right\}_t^T$ of equal length $T = 2^{16}$ time steps, for different initial conditions $\bm x_0^{i}$. We forward-time integrated our system using the Euler-Maruyama scheme and reserved some latent trajectories $\bm x_t^i$ as a held-out validation set and report all our metrics on such examples. We followed the literature \citep{laiz2025, zimmermannContrastiveLearningInverts2022, hyvarinenUnsupervisedFeatureExtraction2016} and parametrized the mixing function $g$ as a 4-layer randomly initialized MLP with explicit control over the condition number for every layer matrix to ensure injectivity and final number of visible dimensions $D = 256$. The complete observation channel \eqref{eq:observation_model} is then obtained by passing the mixing-function's output through a noisy channel. In our experiments we modeled the noise either as a standard additive Gaussian process or as a non-homogeneous Poisson process, both independent across visible dimensions:
\begin{equation}
    \bm{y}_t^a \sim \mathcal{N} \left( g(\bm x_t), \Sigma_\xi \right) \qquad \mathrm{or} \qquad
    \bm{y}_t^a \sim \mathrm{Poisson} \left( \left[ g (\bm  x_t) \right]_+ \Delta t \right),
\label{eq:noise_channels}
\end{equation}
where $\left[\ \cdot\ \right]_+$ is a rectification function to ensure positivity, which we simply take to be a linear rescaling of the output into a pre-defined range followed by a ReLU gate, while we fixed $\Sigma_\xi \equiv \sigma^2 \mathrm{Id}$, $\sigma \in \mathbb{R}$. Different views $a$ are obtained via independent realizations of the same process. We stress that the Poisson construction formally violates the assumptions of Theorem \eqref{thm:identifiability_theorem} since the observation-noise statistics depend on the latent state $\bm x_t$. However, we explored this particular form since it is an ubiquitous model of neural spiking dynamics \citep{pandarinathInferringSingletrialNeural2018} with $\left[ g (\bm x_t)\right]_+$ representing the instantaneous firing rates, making it an important target for our method, and we empirically show that our model still works well in practice. In the cases where an external control signal $\bm u_t$ was present, we treated it as known input to our model and let it learn the appropriate input coupling matrix $C$.

\begin{revisiondiff}
\textbf{Model architecture} In our experiments the encoder network $h$ is a 4-layer MLP with GELU activations and hidden layers of dimension $64$, with appropriate input and output dimensions, i.e. $D$ and $d$. In the experiments with Poisson observations, we included a custom input-normalization layer designed to mitigate the impact of a mean-dependent noise variance introduced by the Poisson sampling. We defined it as the sequential composition of the Freeman-Tukey transform \citep{freeman1950transformations}, a short smoothing convolutional kernel and a final batch-normalization layer. For the dynamics model $\hat f$ we used a polynomial basis (the collection of all monomials of degree up to $\delta$) of sufficient total degree $\delta$ to exactly capture the ground-truth dynamics (\emph{e.g.} $\delta  = 2$ for the case of the Lorenz system). We \rev{generate a global coefficient matrix $\bm\Theta=\operatorname{reshape}(\operatorname{MLP}(\bm c))$ from a learned global source $\bm c \in \mathbb{R}^{32}$, with a 2-layer $\operatorname{MLP}$ with hidden dimensions $64$. Experiment-specific settings are given Appendix~\ref{sec:appendix_baseline_protocol}. We} computed the trajectory roll-outs (both forward and backward in time) for a maximum time horizon $\kappa = 8$ using the RK4 scheme implemented via the \textsc{odeint} function of the \textsc{torchdiffeq} package \citep{torchdiffeq}. We followed the recommendation in \citep{laiz2025} and omitted the scalar potential term since it was shown to have negligible impact on model accuracy while having a significant computational cost.
\end{revisiondiff}

\textbf{Evaluation metrics} The evaluation metrics used to measure the model performance reflect the affine indeterminacy expected from Theorem \eqref{thm:identifiability_theorem}. Given the ground-truth trajectory $\bm x_t$ and the model estimate ${\bm{\hat x}_t}$, we numerically estimate the optimal aligning affine transformation by solving the optimization problem:
\begin{equation}
    L^\star, \bm b^\star \coloneq \arg\min_{L, \bm b} \sum_t^T \parallel \bm x_t - \left( L \bm{\hat x}_t +\bm b \right) \parallel^2_2,
\label{eq:affine_transformation_alignement}
\end{equation}
\begin{revisiondiff}
\rev{The aligned trajectory is $\bm{\hat x}_t^\star=L^\star\bm{\hat x}_t+\bm b^\star$, whose $R^2$ against $\bm x_t$ measures model performance on state recovery. Similarly, since tangent vectors transform with the linear part alone, for $\bm z=L\bm x+\bm b$:}
\begin{equation}
    \rev{\widehat f(\bm z)=L f}\bigl(\rev{L^{-1}(\bm z-\bm b)}\bigr)\rev{,\qquad \widehat C=LC.}
    \label{eq:affine_vector_field}
\end{equation}
We therefore compute the $\mathrm{dyn}R^2$ metric \citep{laiz2025} as the $R^2$ score between $f (\bm x_t)$ and $\hat f^\star (\bm{\hat{x}}_t) = L^\star \hat f(\bm{\hat{x}}_t)$.
\end{revisiondiff}

\textbf{Training details} We trained our model with a combination of AdamW \citep{loshchilov2017decoupled} and Muon \citep{jordan2024muon} optimizers: all hidden-layer matrix parameters were assigned to the Muon optimizers, while the rest was optimized via AdamW. We used a learning rate of $3 \times 10^{-4}$ and a weight decay of $10^{-4}$ (for both optimizers). We trained with a batch size of $4096$ on a single NVIDIA RTX A4000 for $200$ epochs. Depending on the system configuration, the training time was approximately two hours. \rev{Full configuration details are reported in Appendix~\ref{sec:appendix_baseline_protocol}.}

\section{Results}
\label{sec:results}

\begin{table}[tb]
\centering
\renewcommand{\arraystretch}{1.2}
\setlength{\tabcolsep}{7.8pt}
\caption{\textbf{Identification of latent dynamical systems.} We present results for several (non-linear) dynamical systems, spanning several dynamical regimes from chaotic attractors to forced oscillators to point attractors and varying external forcing configuration. For a formal definition of each dynamical system we defer to the Appendix \ref{sec:appendix_dyn_sys}. Each value is the mean across $3$ runs and we report the standard deviation.}
\begin{tabular*}{\textwidth}{@{\extracolsep{\fill}} lcc cccc}
\toprule
& & \multicolumn{3}{c}{Gaussian noise} & \multicolumn{2}{c}{Poisson noise} \\
\cmidrule(lr){3-5}\cmidrule(lr){6-7}
Dynamical System 
& $\bm u_t$ 
& $\sigma$ 
& $\%\ R^2\ (\uparrow)$ 
& $\%\ \mathrm{dyn}R^2\ (\uparrow)$
& $\%\ R^2\ (\uparrow)$ 
& $\%\ \mathrm{dyn}R^2\ (\uparrow)$ \\
\midrule
Duffing 
& \ding{51} & 0.1
& 99.1 $\pm$ 0.9 & 97.8 $\pm$ 2.6
& 98.6 $\pm$ 0.6 & 94.4 $\pm$ 4.6 \\

Lorenz 
& \ding{55} & 1.0
& 93.5 $\pm$ 2.1 & 75.9 $\pm$ 6.6
& 94.4 $\pm$ 1.1 & 82.3 $\pm$ 3.8 \\

FitzHugh--Nagumo 
& \ding{51} & 0.2
& 99.7 $\pm$ 0.1 & 93.8 $\pm$ 0.8
& 99.7 $\pm$ 0.1 & 89.1 $\pm$ 2.9 \\

Winner--Take--All 
& \ding{51} & 0.2
& 95.1 $\pm$ 1.7 & 70.0 $\pm$ 5.4
& 97.4 $\pm$ 0.9 & 72.4 $\pm$ 3.1 \\

Double--Well 
& \ding{51} & 0.2
& 99.3 $\pm$ 0.5 & 93.8 $\pm$ 4.0
& 99.3 $\pm$ 0.4 & 92.7 $\pm$ 1.0 \\

Stuart--Landau 
& \ding{55} & 0.2
& 96.6 $\pm$ 0.8 & 78.4 $\pm$ 4.8
& 98.0 $\pm$ 1.4 & 83.1 $\pm$ 2.2 \\

Heteroclinic 
& \ding{51} & 0.2
& 99.6 $\pm$ 0.2 & 97.5 $\pm$ 0.8
& 99.6 $\pm$ 0.1 & 84.1 $\pm$ 5.5 \\
\midrule
Average 
& & 
& 97.6 $\pm$ 2.6 & 87.4 $\pm$ 11.2 
& 98.0 $\pm$ 2.0 & 85.4 $\pm$ 7.8 \\
\bottomrule
\end{tabular*}
\label{tab:r2_scores}
\end{table}

%
%
%
%

\begin{figure}[tb]
    \centering
    \includegraphics[width=\textwidth]{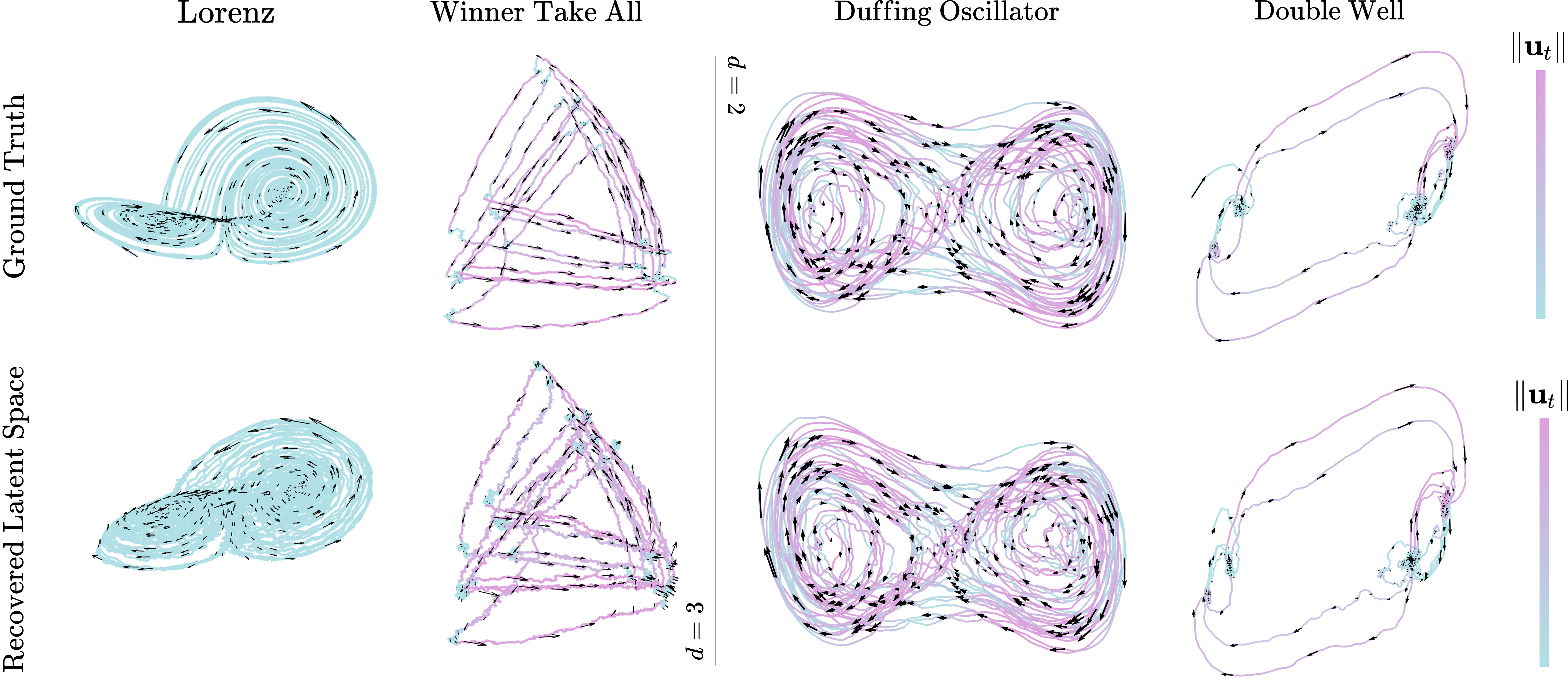}
    \caption{\textbf{Dynamical system recovery via multi-view contrastive learning}. Example phase space portraits of different dynamical systems (top-row) with corresponding affine-aligned recovered systems (bottom row) for either $d=3$ (left) or $d=2$ (right). Dynamical system trajectories are color-coded based on the external forcing magnitude $\parallel \bm u_t \parallel$. All reported systems are observed via a Poisson observation channel.}
    \label{fig:example_dynamical_systems}
\end{figure}

\begin{figure}[tb]
    \centering
    \includegraphics[width=\textwidth]{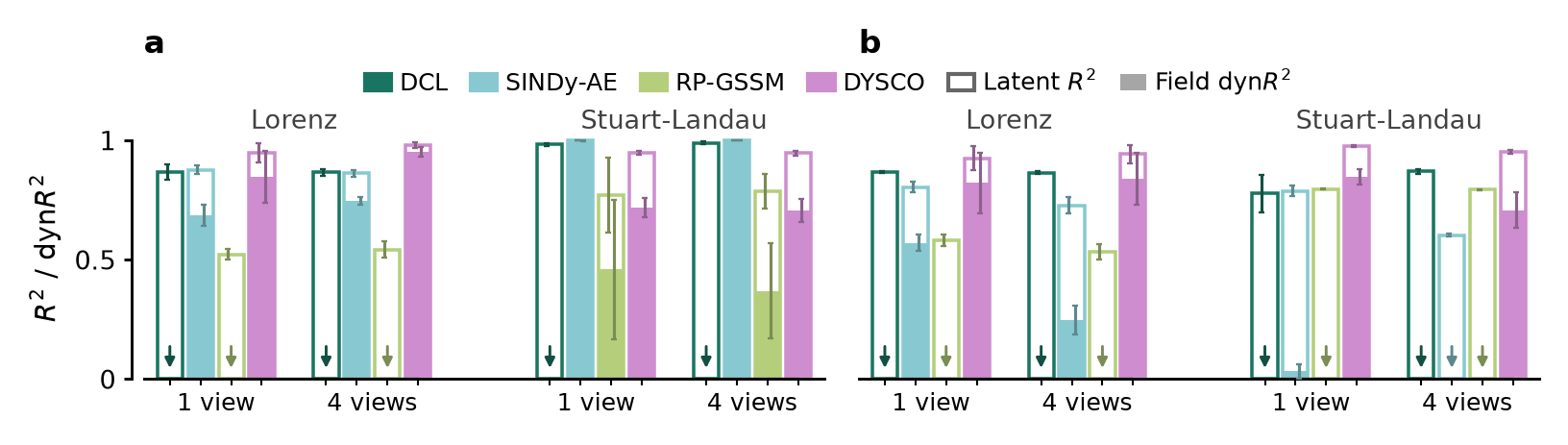}
    \caption{\textbf{Controlled baseline comparisons.} Latent $R^2$ (outlined bars) and field dyn$R^2$ (filled bars) for Lorenz and Stuart--Landau, with one or four views. (\textbf{a}) Gaussian noise, $\sigma=1.0$ for Lorenz and $0.2$ for Stuart--Landau. (\textbf{b}) Poisson noise. Values are mean $\pm$ SEM over the three independent. For $\mathrm{dyn}R^2$ we reported $\max \left(0, \mathrm{dynR^2} \right)$, with arrows denoting clipped values.}
    \label{fig:baseline_comparisons}
\end{figure}

\begin{revisiondiff}
\textbf{Dynamical Systems} In this section we report the identification results for several different dynamical systems. For a complete formal description of the dynamical systems considered in this study we refer to Appendix \ref{sec:appendix_dyn_sys}. In Table \ref{tab:r2_scores} we have reported the results for both $R^2$ and $\mathrm{dyn}R^2$ metrics for the dynamical systems explored in this work, while Figure \ref{fig:example_dynamical_systems} reports some phase-portrait visual examples. Indeed, for a broad class of systems, spanning different dynamical regimes, external-forcing and noise configurations, our algorithm consistently achieves very high latent trajectory accuracies and good-to-very-good flow field accuracies. We note how even for the Winner-Take-All system where the $\mathrm{dyn}R^2$ metric shows average results, the phase portrait (Figure \ref{fig:example_dynamical_systems}, second column) in fact shows excellent agreement with the ground truth. We take this fact to indicate how the metric is potentially affected by the many slow-velocity attractors, where the flow field strength gets overshadowed by the system and observational noise.

\rev{\textbf{Controlled baseline comparisons} We compare DYSCO with DCL \citep{laiz2025}, SINDy-AE \citep{championDatadrivenDiscoveryCoordinates2019} and RP-GSSM \citep{hromadka2025maximum} on Lorenz and Stuart--Landau (Figure~\ref{fig:baseline_comparisons}). At each view count, methods receive the same raw observations, \textsc{DYSCO} additionally uses their alignment, while baselines treat views as additional trajectories. DCL uses switching-linear dynamics and assumes clean nonlinear observations, an assumption violated in our noisy-observation setting. DYSCO and SINDy-AE use polynomial libraries compatible with the benchmark fields. While SINDy-AE performs particularly well on Gaussian Stuart--Landau, DYSCO provides strong field recovery, with a leading gap particularly marked under Poisson noise, where all other baselines struggle. These results confirm that multi-view consistency is one critical mechanism enabling identification under realistic noisy observations.}
\end{revisiondiff}

\begin{figure}[tb]
    \centering
    \includegraphics[width=\textwidth]{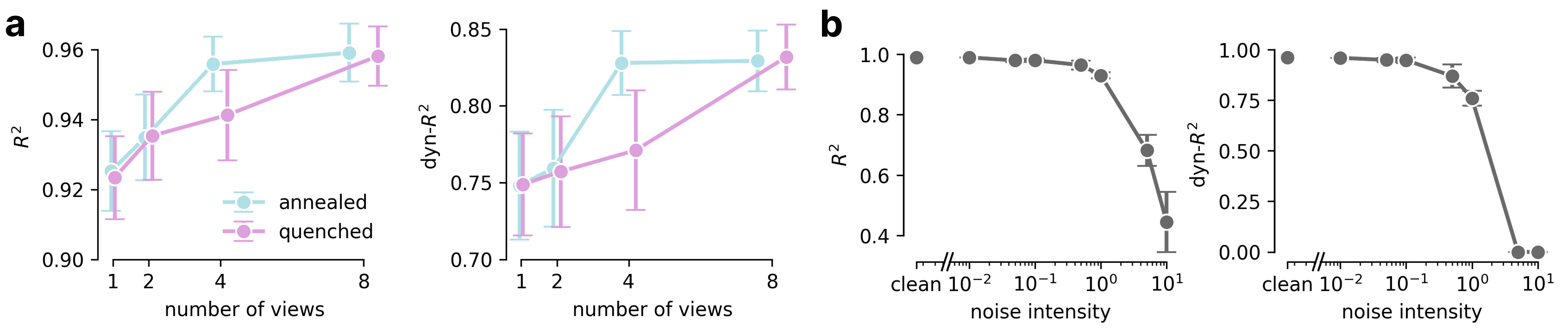}
    \caption{\textbf{Ablation results}. (\textbf{a}) Results for impact of available number of views on the final trajectory (left) and flow-field (right) model performances for the Lorenz dynamical system (observed via a Gaussian channel with standard deviation $\sigma = 1$) and for the two noise condition annealed (cyan) and quenched (plum). Markers are averages over $3$ repetitions and error bars are the standard error of the mean. (\textbf{b}) Impact of the observation noise intensity $\sigma$ on model performances for the Lorenz dynamical systems observed via a Gaussian channel of increasing variance $\sigma^2$ and for $| V |=2$. The clean condition corresponds to the case $\sigma = 0$ (\emph{i.e.} no observational noise).}
    \label{fig:ablation_results}
\end{figure}

\textbf{Symbolic recovery} While Theorem \ref{thm:identifiability_theorem} guarantees identification of the ground-truth dynamical system up to an affine transformation, exact symbolic recovery is precluded by this gauge freedom. The coefficient representation $\bm \Theta$ of the latent flow field is gauge-dependent, since rotations or shifts in latent space mix the components of $\bm \Theta$ in a structured way. For the practically important case of a polynomial basis of degree $\delta$, however, the basis is closed under affine transformations $\bm x_t' =\mathcal{A}\bm x_t \equiv L\bm x_t + \bm b$, and the induced action on the coefficients can be written explicitly:
\begin{equation}
    \bm \Theta' = L
    \bm \Theta
    T_{\mathcal A^{-1}}^{(\delta)},
\label{eq:coefficient_change_basis}
\end{equation}
where $T^\delta_{\mathcal{A}^{-1}}$ is a coefficient transport matrix defined via the relation:
\begin{equation}
    \Xi_\delta \left(L^{-1}\left( \bm x_t - \bm b \right) \right) = T^\delta_{\mathcal{A}^{-1}} \Xi_\delta \left( \bm x_t \right).
\label{eq:coefficient_transport_matrix}
\end{equation}
This characterization defines the orbit of valid coefficient representations and opens the door to downstream symbolic regression:
under the assumption that the underlying system satisfies \emph{symbolic simplicity} \citep{bruntonDiscoveringGoverningEquations2016}, we can attempt symbolic identification by seeking the sparsest representative within the orbit $\bm \Theta^\star$:
\begin{equation}
    \bm \Theta^\star = \arg\min_{L, \bm b} \parallel L \bm  \Theta T^{(\delta)}_{\mathcal{A}^{-1}} \parallel_0.
\label{eq:sparsity_maximisation}
\end{equation}
We offer a proof of concept application of this technique in Appendix \ref{sec:appendix_symbolic_recovery_affine_orbits}.

\begin{figure}[tb]
    \centering
    \includegraphics[width=\linewidth]{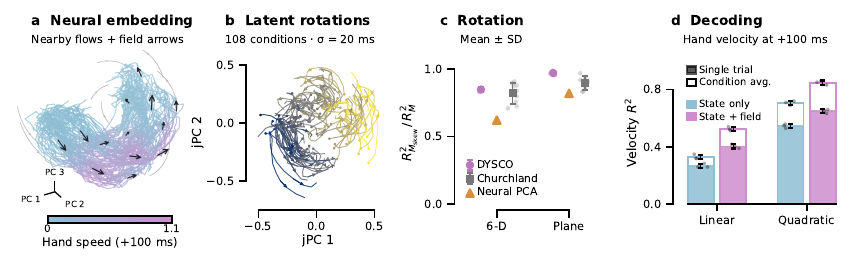}
    \caption{\textbf{DYSCO on motor-cortical recordings.} (\textbf{a}) Held-out condition means in 3D PCs, colored by hand speed at $+100$ ms. Dark arrows: projected field tangents ($40$ ms); grey curves: nearby 6D flows ($u=0$, $\pm80$ ms). (\textbf{b}) Training-fitted jPCA trajectories of DYSCO latents, colored by starting jPC1; we applied a light Gaussian smoothing with $\sigma=20$ ms, arrows indicate temporal progression. (\textbf{c}) Skew/full $R^2$ ratios with neural-PCA and \citep{churchland2012neural} controls. (\textbf{d}) Linear and quadratic ridge decoding of hand velocity from state alone and state plus field. Filled bars: single trials; outlines: condition averages. DYSCO values: $3$ seeds, mean $\pm$ SD.}
    \label{fig:neural_analysis}
\end{figure}
\begin{revision}
\textbf{Motor-cortical recordings} To assess the applicability of DYSCO on real neural data, we considered the MC\_Maze dataset \citep{pei2021neural}, comprising $137$ neurons recorded from macaque motor cortex during a reaching task \citep{churchland2012neural}. We regarded repeated movement-aligned trials as approximate views of a shared condition-dependent trajectory and trained a six-dimensional DYSCO model with a quadratic field and the observed go cue as an external input (Figure~\ref{fig:neural_analysis}a). The resulting latent space exhibited rotational organization (Figure~\ref{fig:neural_analysis}a,b), with skew/full fit ratios in general agreement with the published neural values (Figure~\ref{fig:neural_analysis}c). To inspect the learned equations, we chose a fixed basis whose first two coordinates span the training jPCA plane. The corresponding linear field components showed a clear rotational structure:
\begin{equation}
\begin{aligned}
    \dot q_1 &= -0.43q_1-4.67 q_2 + r_1 (\bm q,u),\\
    \dot q_2 &= \phantom{-}5.59 q_1 + 0.46 q_2+r_2(\bm q,u),
\end{aligned}
\label{eq:neural_rotating_field}
\end{equation}
where $r_1, r_2$ collect the remaining terms (constant, linear couplings to the other four coordinates, quadratic terms and observed-input response). Appendix~\ref{sec:appendix_neural} details the coordinate choice, full linear matrix and approximation errors. Finally, including the learned field improved single-trial hand-velocity decoding for both linear ($R^2=0.402\pm0.015$ versus $0.267\pm0.014$) and quadratic readouts ($0.649\pm0.015$ versus $0.545\pm0.015$; Figure~\ref{fig:neural_analysis}d).
\end{revision}

\textbf{Ablation results} We investigated the properties of our learning framework via two ablation studies. We first gauged the impact of the number of total available views $|V|$ on the final model performances. While Theorem \ref{thm:identifiability_theorem} guarantees identifiability asymptotically for $ V  \to \infty$, in practice we intuitively expected that a richer set of independent views would allow our model to achieve progressively better results. We focused on the Lorenz system and trained the same encoder and dynamics model on a fixed dataset of latent trajectories for a progressively increasing number of available views $a \in \left\{1, 2, 4, 8 \right\}$. We explored both possible noise conditions: \emph{annealed} and \emph{quenched}, corresponding, respectively, to the case where the observation noise for a given latent state $\bm x_t$ is resampled at every training iteration or not. The quenched configuration we argue is conceptually closer to the experimental reality where data acquisition is performed once, yielding frozen noise realizations. We performed our experiment for a Gaussian observation model with noise intensity $\sigma = 1$ and report our results in Figure \ref{fig:ablation_results}a. Indeed we observed that model performance improved monotonically as a function of the number of views, with the annealed configuration offering minor performance advantages for small numbers of views, which is expected due to the richer resampling of the noise. Interestingly, even with a single available view, the model achieved good results: we speculate that for the particular case of a chaotic attractor where the system explores arbitrarily close trajectories in phase space, even a single view effectively allows for dense resampling of neighboring latent states.

We then explored how robust our system was for an increasingly noisy observation channel. We fixed $ V  = 2$ and trained our model on the Lorenz system with $\bm \xi_t^a \sim \mathcal{N}(0, \sigma^2\mathrm{Id})$ Gaussian, for increasing noise intensity $\sigma$. Our results are showcased in Figure \ref{fig:ablation_results}b, where the model consistently achieved good performances up to $\sigma^\ast \simeq 1$, which translates to a threshold signal-to-noise ratio of $\mathrm{SNR}^\ast \simeq 4 \mathrm{dB}$. Interestingly, Figure \ref{fig:ablation_results}b highlights a small region where trajectory recovery is still possible, but the dynamics is not. This observation is consistent with prior work \citep{laiz2025} which illustrated how, when noise dominates the dynamics, even simple contrastive objectives can recover meaningful latent structure while explicit dynamics modeling fails. This delineation helps clarify when standard contrastive learning suffices and when more structured approaches, such as ours, are required. Lastly, we explored the impact of the integration horizon $\kappa$ on model performance in Appendix \ref{sec:appendix_ablation_integration_horizon}, \rev{while an overcomplete-library diagnostic and forward-only objective ablation appear in Appendix~\ref{sec:appendix_objective_library}}.

\vspace{-6pt}
\section{Discussion}
\label{sec:discussion}
\vspace{-2pt}

In this work we have introduced a \rev{multi-view contrastive} inference algorithm for the (symbolic) identification of latent dynamical systems under noisy observations. We showed how our model consistently achieves high quality recovery of both latent trajectories and dynamics under significant observational noise. From a theoretical perspective, our results reinforce and extend the emerging view that contrastive learning implicitly performs inversion of the data-generating process \citep{zimmermannContrastiveLearningInverts2021}. Compared to prior frameworks such as \textsc{\rev{DCL}} \citep{laiz2025}, our approach demonstrates that identifiability can be retained even under noisy nonlinear observation channels by leveraging the implicit denoising constraint imposed by multi-view consistency, thereby relaxing the idealized observation assumptions typically required in theory. Another important implication concerns the interplay between representation learning and scientific modeling. We opted for explicit symbolic modeling of the dynamics. Our results suggest that one can recover models that are both expressive and interpretable. In this sense, our framework provides a bridge between self-supervised learning and symbolic regression. \rev{Lastly}, we argue that our setup is relevant for neuroscience, where measurements are inherently noisy but often trial-structured, making the assumption of repeated observations of a shared latent state natural. Achieving strong identifiability in experimentally realistic conditions thus offers a precious tool in the pursuit of understanding the algorithmic function of neural circuits via the lens of dynamical systems \citep{vyasComputationNeuralPopulation2020}. \rev{In Fig.\ref{fig:neural_analysis} we have provided initial evidence toward this goal, illustrating how \textsc{dysco} recovered a known rotational structure in motor-cortical recordings \citep{churchland2012neural}}. This, in turn, suggests that structured contrastive learning approaches might offer a viable path toward more interpretable, dynamics-level descriptions of neural computation.

\textbf{Limitations}
A first limitation concerns the choice of latent dimensionality $d$ and symbolic basis $\Xi_f$, which are assumed to be known. In practice, both quantities must be estimated from data, and poor choices may hinder accurate system identification (\emph{e.g.}, the chosen basis is not expressive enough to capture the underlying dynamics, \rev{or it's overcomplete slowing convergence, see Appendix \ref{sec:appendix_objective_library}}). Nevertheless, this limitation can be mitigated in practice. Estimating the intrinsic dimensionality of data is a well-studied problem, with a range of established estimators available \citep{levina2004maximum, facco2017estimating, muratore2022prune}. Moreover, polynomial bases provide a flexible approximation class (\emph{e.g.}, via Taylor expansions) while allowing explicit control over model complexity, making them a practical choice in many settings. Lastly, our current implementation disregards the contribution of the scalar potential $\alpha$: this term can potentially become important for systems with non-uniform marginals, for example for systems with different fast-slow regimes. Future work should explore how to learn this term jointly and effectively.

\clearpage
\begingroup
\makeatletter
\def\crnewrefs{schimel2022ilqr,brenner2024multimodal,hromadka2025maximum,churchland2012neural,pei2021neural}
\let\crbibitem\bibitem
\renewcommand{\bibitem}[2][]{\color{black}\@for\crkey:=\crnewrefs\do{\edef\crtarget{#2}\ifx\crkey\crtarget\color{cameraready}\fi}\crbibitem[#1]{#2}}
\makeatother
\bibliography{bibliography}

@inproceedings{laiz2025,
  title = {Self-Supervised Contrastive Learning Performs Non-Linear System Identification},
  booktitle = {ICLR},
  author = {Laiz, Rodrigo Gonzalez and Schmidt, Tobias and Schneider, Steffen},
  year = 2025,
  pages = {38},
  langid = {english},
}

@article{mathis2026joint,
  author    = {Mathis, M. W. and Mathis, A.},
  title     = {Joint modelling of brain and behaviour dynamics with artificial intelligence},
  journal   = {Nature Reviews Neuroscience},
  volume    = {27},
  pages     = {87--100},
  year      = {2026},
  doi       = {10.1038/s41583-025-00996-1},
  url       = {https://doi.org/10.1038/s41583-025-00996-1}
}

@article{bruntonDiscoveringGoverningEquations2016,
  title = {Discovering Governing Equations from Data by Sparse Identification of Nonlinear Dynamical Systems},
  author = {Brunton, Steven L. and Proctor, Joshua L. and Kutz, J. Nathan},
  year = 2016,
  month = apr,
  journal = {Proceedings of the National Academy of Sciences},
  volume = {113},
  number = {15},
  pages = {3932--3937},
  issn = {0027-8424, 1091-6490},
  doi = {10.1073/pnas.1517384113},
  urldate = {2025-10-30},
  langid = {english}
}

@misc{halvaDisentanglingIdentifiableFeatures2021,
  title = {Disentangling {{Identifiable Features}} from {{Noisy Data}} with {{Structured Nonlinear ICA}}},
  author = {H{\"a}lv{\"a}, Hermanni and Corff, Sylvain Le and Leh{\'e}ricy, Luc and So, Jonathan and Zhu, Yongjie and Gassiat, Elisabeth and Hyvarinen, Aapo},
  year = 2021,
  month = oct,
  number = {arXiv:2106.09620},
  eprint = {2106.09620},
  primaryclass = {stat},
  publisher = {arXiv},
  doi = {10.48550/arXiv.2106.09620},
  urldate = {2026-02-17},
  archiveprefix = {arXiv}
}

@misc{zimmermannContrastiveLearningInverts2022,
  title = {Contrastive {{Learning Inverts}} the {{Data Generating Process}}},
  author = {Zimmermann, Roland S. and Sharma, Yash and Schneider, Steffen and Bethge, Matthias and Brendel, Wieland},
  year = 2022,
  month = apr,
  number = {arXiv:2102.08850},
  eprint = {2102.08850},
  primaryclass = {cs},
  publisher = {arXiv},
  doi = {10.48550/arXiv.2102.08850},
  urldate = {2025-05-25},
  archiveprefix = {arXiv}
}

@inproceedings{hyvarinenUnsupervisedFeatureExtraction2016,
  title = {Unsupervised {{Feature Extraction}} by {{Time-Contrastive Learning}} and {{Nonlinear ICA}}},
  booktitle = {Advances in {{Neural Information Processing Systems}}},
  author = {Hyvarinen, Aapo and Morioka, Hiroshi},
  year = 2016,
  volume = {29},
  publisher = {Curran Associates, Inc.},
  urldate = {2026-04-21}
}

@article{pandarinathInferringSingletrialNeural2018,
  title = {Inferring Single-Trial Neural Population Dynamics Using Sequential Auto-Encoders},
  author = {Pandarinath, Chethan and O'Shea, Daniel J. and Collins, Jasmine and Jozefowicz, Rafal and Stavisky, Sergey D. and Kao, Jonathan C. and Trautmann, Eric M. and Kaufman, Matthew T. and Ryu, Stephen I. and Hochberg, Leigh R. and Henderson, Jaimie M. and Shenoy, Krishna V. and Abbott, L. F. and Sussillo, David},
  year = 2018,
  month = oct,
  journal = {Nature Methods},
  volume = {15},
  number = {10},
  pages = {805--815},
  publisher = {Nature Publishing Group},
  issn = {1548-7105},
  doi = {10.1038/s41592-018-0109-9},
  urldate = {2025-05-13},
  copyright = {2018 The Author(s), under exclusive licence to Springer Nature America, Inc.},
  langid = {english}
}

@article{freeman1950transformations,
  title={Transformations related to the angular and the square root},
  author={Freeman, Murray F and Tukey, John W},
  journal={The annals of mathematical statistics},
  pages={607--611},
  year={1950},
  publisher={JSTOR}
}

@misc{jordan2024muon,
  author       = {Keller Jordan and Yuchen Jin and Vlado Boza and You Jiacheng and
                  Franz Cesista and Laker Newhouse and Jeremy Bernstein},
  title        = {Muon: An optimizer for hidden layers in neural networks},
  year         = {2024},
  url          = {https://kellerjordan.github.io/posts/muon/}
}

@article{loshchilov2017decoupled,
  title={Decoupled weight decay regularization},
  author={Loshchilov, Ilya and Hutter, Frank},
  journal={arXiv preprint arXiv:1711.05101},
  year={2017}
}

@misc{torchdiffeq,
	author={Chen, Ricky T. Q.},
	title={torchdiffeq},
	year={2018},
	url={https://github.com/rtqichen/torchdiffeq},
}

@article{oord2018representation,
  title={Representation learning with contrastive predictive coding},
  author={Oord, Aaron van den and Li, Yazhe and Vinyals, Oriol},
  journal={arXiv preprint arXiv:1807.03748},
  year={2018}
}

@article{schneiderLearnableLatentEmbeddings2023,
  title = {Learnable Latent Embeddings for Joint Behavioural and Neural Analysis},
  author = {Schneider, Steffen and Lee, Jin Hwa and Mathis, Mackenzie Weygandt},
  year = 2023,
  month = may,
  journal = {Nature},
  volume = {617},
  number = {7960},
  pages = {360--368},
  publisher = {Nature Publishing Group},
  issn = {1476-4687},
  doi = {10.1038/s41586-023-06031-6},
  urldate = {2024-11-05},
  copyright = {2023 The Author(s)},
  langid = {english}
}

@misc{dascoliODEFormerSymbolicRegression2023,
  title = {{{ODEFormer}}: {{Symbolic Regression}} of {{Dynamical Systems}} with {{Transformers}}},
  shorttitle = {{{ODEFormer}}},
  author = {{d'Ascoli}, St{\'e}phane and Becker, S{\"o}ren and Mathis, Alexander and Schwaller, Philippe and Kilbertus, Niki},
  year = 2023,
  month = oct,
  number = {arXiv:2310.05573},
  eprint = {2310.05573},
  publisher = {arXiv},
  urldate = {2024-11-05},
  archiveprefix = {arXiv}
}

@article{championDatadrivenDiscoveryCoordinates2019,
  title = {Data-Driven Discovery of Coordinates and Governing Equations},
  author = {Champion, Kathleen and Lusch, Bethany and Kutz, J. Nathan and Brunton, Steven L.},
  year = 2019,
  month = nov,
  journal = {Proceedings of the National Academy of Sciences},
  volume = {116},
  number = {45},
  pages = {22445--22451},
  publisher = {Proceedings of the National Academy of Sciences},
  doi = {10.1073/pnas.1906995116},
  urldate = {2025-07-03}
}

@inproceedings{hyvarinenNonlinearICAUsing2019,
  title = {Nonlinear {{ICA Using Auxiliary Variables}} and {{Generalized Contrastive Learning}}},
  booktitle = {Proceedings of the {{Twenty-Second International Conference}} on {{Artificial Intelligence}} and {{Statistics}}},
  author = {Hyvarinen, Aapo and Sasaki, Hiroaki and Turner, Richard},
  year = 2019,
  month = apr,
  pages = {859--868},
  publisher = {PMLR},
  issn = {2640-3498},
  urldate = {2026-04-27},
  langid = {english}
}

@article{hyvarinenNonlinearIndependentComponent2023,
  title = {Nonlinear Independent Component Analysis for Principled Disentanglement in Unsupervised Deep Learning},
  author = {Hyv{\"a}rinen, Aapo and Khemakhem, Ilyes and Morioka, Hiroshi},
  year = 2023,
  month = oct,
  journal = {Patterns},
  volume = {4},
  number = {10},
  publisher = {Elsevier},
  issn = {2666-3899},
  doi = {10.1016/j.patter.2023.100844},
  urldate = {2026-02-25},
  langid = {english},
  pmid = {37876900}
}

@inproceedings{zimmermannContrastiveLearningInverts2021,
  title = {Contrastive {{Learning Inverts}} the {{Data Generating Process}}},
  booktitle = {Proceedings of the 38th {{International Conference}} on {{Machine Learning}}},
  author = {Zimmermann, Roland S. and Sharma, Yash and Schneider, Steffen and Bethge, Matthias and Brendel, Wieland},
  year = 2021,
  month = jul,
  pages = {12979--12990},
  publisher = {PMLR},
  issn = {2640-3498},
  urldate = {2025-05-23},
  langid = {english}
}

@misc{shahImprovedInterpretabilityLFADS2025,
  title = {Improved Interpretability in {{LFADS}} Models Using a Learned, Context-Dependent per-Trial Bias},
  author = {Shah, Nishal P. and Krasa, Benyamin Abramovich and Kunz, Erin and Hahn, Nick and Kamdar, Foram and Avansino, Donald and Hochberg, Leigh R. and Henderson, Jaimie M. and Sussillo, David},
  year = 2025,
  month = oct,
  primaryclass = {New Results},
  pages = {2025.10.03.680303},
  publisher = {bioRxiv},
  issn = {2692-8205},
  doi = {10.1101/2025.10.03.680303},
  urldate = {2025-10-21},
  archiveprefix = {bioRxiv},
  chapter = {New Results},
  copyright = {\copyright{} 2025, Posted by Cold Spring Harbor Laboratory. This pre-print is available under a Creative Commons License (Attribution 4.0 International), CC BY 4.0, as described at http://creativecommons.org/licenses/by/4.0/},
  langid = {english}
}

@article{pandarinathInferringSingletrialNeural2018a,
  title = {Inferring Single-Trial Neural Population Dynamics Using Sequential Auto-Encoders},
  author = {Pandarinath, Chethan and O'Shea, Daniel J. and Collins, Jasmine and Jozefowicz, Rafal and Stavisky, Sergey D. and Kao, Jonathan C. and Trautmann, Eric M. and Kaufman, Matthew T. and Ryu, Stephen I. and Hochberg, Leigh R. and Henderson, Jaimie M. and Shenoy, Krishna V. and Abbott, L. F. and Sussillo, David},
  year = 2018,
  month = oct,
  journal = {Nature Methods},
  volume = {15},
  number = {10},
  pages = {805--815},
  publisher = {Nature Publishing Group},
  issn = {1548-7105},
  doi = {10.1038/s41592-018-0109-9},
  urldate = {2026-04-27},
  copyright = {2018 The Author(s), under exclusive licence to Springer Nature America, Inc.},
  langid = {english}
}

@article{klishinStatisticalMechanicsDynamical2025,
  title = {Statistical Mechanics of Dynamical System Identification},
  author = {Klishin, Andrei A. and Bakarji, Joseph and Kutz, J. Nathan and Manohar, Krithika},
  year = 2025,
  month = aug,
  journal = {Physical Review Research},
  volume = {7},
  number = {3},
  pages = {033181},
  publisher = {American Physical Society},
  doi = {10.1103/4d98-tdlp},
  urldate = {2025-08-26}
}

@inproceedings{kimInferringLatentDynamics2021,
  title = {Inferring {{Latent Dynamics Underlying Neural Population Activity}} via {{Neural Differential Equations}}},
  booktitle = {Proceedings of the 38th {{International Conference}} on {{Machine Learning}}},
  author = {Kim, Timothy D. and Luo, Thomas Z. and Pillow, Jonathan W. and Brody, Carlos D.},
  year = 2021,
  month = jul,
  pages = {5551--5561},
  publisher = {PMLR},
  issn = {2640-3498},
  urldate = {2026-04-27},
  langid = {english}
}

@article{bastiani2025diffusion,
  title={Diffusion-Based Symbolic Regression},
  author={Bastiani, Zachary and Kirby, Robert M and Hochhalter, Jacob and Zhe, Shandian},
  journal={arXiv preprint arXiv:2505.24776},
  year={2025}
}

@inproceedings{henaff2020data,
  title={Data-efficient image recognition with contrastive predictive coding},
  author={Henaff, Olivier},
  booktitle={International conference on machine learning},
  pages={4182--4192},
  year={2020},
  organization={PMLR}
}

@article{schneider2019wav2vec,
  title={wav2vec: Unsupervised pre-training for speech recognition},
  author={Schneider, Steffen and Baevski, Alexei and Collobert, Ronan and Auli, Michael},
  journal={arXiv preprint arXiv:1904.05862},
  year={2019}
}

@article{facco2017estimating,
  title={Estimating the intrinsic dimension of datasets by a minimal neighborhood information},
  author={Facco, Elena and d’Errico, Maria and Rodriguez, Alex and Laio, Alessandro},
  journal={Scientific reports},
  volume={7},
  number={1},
  pages={12140},
  year={2017},
  publisher={Nature Publishing Group UK London}
}

@article{levina2004maximum,
  title={Maximum likelihood estimation of intrinsic dimension},
  author={Levina, Elizaveta and Bickel, Peter},
  journal={Advances in neural information processing systems},
  volume={17},
  year={2004}
}

@article{marr1982vision,
  title={Vision: A computational investigation into the human representation and processing of visual information},
  author={Marr, David},
  journal={(No Title)},
  year={1982}
}

@misc{versteegComputationthroughDynamicsBenchmarkSimulated2025,
  title = {Computation-through-{{Dynamics Benchmark}}: {{Simulated}} Datasets and Quality Metrics for Dynamical Models of Neural Activity},
  shorttitle = {Computation-through-{{Dynamics Benchmark}}},
  author = {Versteeg, Christopher and McCart, Jonathan D. and Ostrow, Mitchell and Zoltowski, David M. and Washington, Clayton B. and Driscoll, Laura and Codol, Olivier and Michaels, Jonathan A. and Linderman, Scott W. and Sussillo, David and Pandarinath, Chethan},
  year = 2025,
  month = feb,
  primaryclass = {New Results},
  pages = {2025.02.07.637062},
  publisher = {bioRxiv},
  doi = {10.1101/2025.02.07.637062},
  urldate = {2025-02-14},
  archiveprefix = {bioRxiv},
  chapter = {New Results},
  copyright = {\copyright{} 2025, Posted by Cold Spring Harbor Laboratory. This pre-print is available under a Creative Commons License (Attribution 4.0 International), CC BY 4.0, as described at http://creativecommons.org/licenses/by/4.0/},
  langid = {english}
}

@article{vyasComputationNeuralPopulation2020,
  title = {Computation {{Through Neural Population Dynamics}}},
  author = {Vyas, Saurabh and Golub, Matthew D. and Sussillo, David and Shenoy, Krishna V.},
  year = 2020,
  month = jul,
  journal = {Annual Review of Neuroscience},
  volume = {43},
  number = {1},
  pages = {249--275},
  issn = {0147-006X, 1545-4126},
  doi = {10.1146/annurev-neuro-092619-094115},
  urldate = {2025-03-19},
  langid = {english}
}

@article{sussilloOpeningBlackBox2013,
  title = {Opening the {{Black Box}}: {{Low-Dimensional Dynamics}} in {{High-Dimensional Recurrent Neural Networks}}},
  shorttitle = {Opening the {{Black Box}}},
  author = {Sussillo, David and Barak, Omri},
  year = 2013,
  month = mar,
  journal = {Neural Computation},
  volume = {25},
  number = {3},
  pages = {626--649},
  issn = {0899-7667},
  doi = {10.1162/NECO_a_00409},
  urldate = {2026-04-09}
}

@article{mastrogiuseppe2018linking,
  title={Linking connectivity, dynamics, and computations in low-rank recurrent neural networks},
  author={Mastrogiuseppe, Francesca and Ostojic, Srdjan},
  journal={Neuron},
  volume={99},
  number={3},
  pages={609--623},
  year={2018},
  publisher={Elsevier}
}

@article{muratore2022prune,
  title={Prune and distill: similar reformatting of image information along rat visual cortex and deep neural networks},
  author={Muratore, Paolo and Tafazoli, Sina and Piasini, Eugenio and Laio, Alessandro and Zoccolan, Davide},
  journal={Advances in Neural Information Processing Systems},
  volume={35},
  pages={30206--30218},
  year={2022}
}

@misc{yangDeepHoyerLearningSparser2020,
  title = {{{DeepHoyer}}: {{Learning Sparser Neural Network}} with {{Differentiable Scale-Invariant Sparsity Measures}}},
  shorttitle = {{{DeepHoyer}}},
  author = {Yang, Huanrui and Wen, Wei and Li, Hai},
  year = 2020,
  month = jan,
  number = {arXiv:1908.09979},
  eprint = {1908.09979},
  primaryclass = {cs},
  publisher = {arXiv},
  doi = {10.48550/arXiv.1908.09979},
  urldate = {2025-08-25},
  archiveprefix = {arXiv}
}

@misc{louizosLearningSparseNeural2018a,
  title = {Learning {{Sparse Neural Networks}} through L0 {{Regularization}}},
  author = {Louizos, Christos and Welling, Max and Kingma, Diederik P.},
  year = 2018,
  month = jun,
  number = {arXiv:1712.01312},
  eprint = {1712.01312},
  primaryclass = {stat},
  publisher = {arXiv},
  doi = {10.48550/arXiv.1712.01312},
  urldate = {2025-07-30},
  archiveprefix = {arXiv}
}

@misc{kolbDifferentiableSparsity$D$Gating2025,
  title = {Differentiable {{Sparsity}} via D-Gating: {{Simple}} and {{Versatile Structured Penalization}}},
  shorttitle = {Differentiable {{Sparsity}} via \${{D}}\$-{{Gating}}},
  author = {Kolb, Chris and Frost, Laetitia and Bischl, Bernd and R{\"u}gamer, David},
  year = 2025,
  month = oct,
  number = {arXiv:2509.23898},
  eprint = {2509.23898},
  primaryclass = {cs},
  publisher = {arXiv},
  doi = {10.48550/arXiv.2509.23898},
  urldate = {2025-12-04},
  archiveprefix = {arXiv}
}

@inproceedings{wang2020understanding,
  title={Understanding contrastive representation learning through alignment and uniformity on the hypersphere},
  author={Wang, Tongzhou and Isola, Phillip},
  booktitle={International conference on machine learning},
  pages={9929--9939},
  year={2020},
  organization={PMLR}
}

@inproceedings{schimel2022ilqr,
  title = {{iLQR-VAE}: Control-Based Learning of Input-Driven Dynamics with Applications to Neural Data},
  author = {Schimel, Marine and Kao, Ta-Chu and Jensen, Kristopher T. and Hennequin, Guillaume},
  booktitle = {International Conference on Learning Representations},
  year = {2022},
  url = {https://openreview.net/forum?id=wRODLDHaAiW}
}

@inproceedings{brenner2024multimodal,
  title = {Integrating Multimodal Data for Joint Generative Modeling of Complex Dynamics},
  author = {Brenner, Manuel and Hess, Florian and Koppe, Georgia and Durstewitz, Daniel},
  booktitle = {Proceedings of the 41st International Conference on Machine Learning},
  series = {Proceedings of Machine Learning Research},
  volume = {235}, pages = {4482--4516}, year = {2024}, publisher = {PMLR},
  url = {https://proceedings.mlr.press/v235/brenner24a.html}
}

@article{hromadka2025maximum,
  title = {Maximum Likelihood Learning of Latent Dynamics Without Reconstruction},
  author = {Hromadka, Samo and Biegun, Kai and Fox, Lior and Heald, James and Sahani, Maneesh},
  journal = {arXiv preprint arXiv:2505.23569}, year = {2025},
  doi = {10.48550/arXiv.2505.23569}
}

@article{churchland2012neural,
  title = {Neural Population Dynamics During Reaching},
  author = {Churchland, Mark M. and Cunningham, John P. and Kaufman, Matthew T. and Foster, Justin D. and Nuyujukian, Paul and Ryu, Stephen I. and Shenoy, Krishna V.},
  journal = {Nature}, volume = {487}, pages = {51--56}, year = {2012},
  doi = {10.1038/nature11129}
}

@inproceedings{pei2021neural,
  title = {Neural Latents Benchmark '21: Evaluating Latent Variable Models of Neural Population Activity},
  author = {Pei, Felix and Ye, Joel and Zoltowski, David and Wu, Anqi and Chowdhury, Raeed H. and Sohn, Hansem and O'Doherty, Joseph E. and Shenoy, Krishna V. and Kaufman, Matthew T. and Churchland, Mark and Jazayeri, Mehrdad and Miller, Lee E. and Pillow, Jonathan and Park, Il Memming and Dyer, Eva L. and Pandarinath, Chethan},
  booktitle = {Proceedings of the Neural Information Processing Systems Track on Datasets and Benchmarks},
  volume = {1}, year = {2021},
  url = {https://datasets-benchmarks-proceedings.neurips.cc/paper/2021/hash/979d472a84804b9f647bc185a877a8b5-Abstract-round2.html}
}
\endgroup

\clearpage
\appendix
\begin{center}
    {\huge \bfseries Supplementary Material\par}
\end{center}
\vspace{1em}
\begin{revisiondiff}
\section{Proof of Theoretical Results}
\label{sec:appendix_theorem_proofs}

In this section we state the claim of the main theoretical result and provide a full proof. We then provide two remarks discussing the practical implementation choices for \textsc{DYSCO}.

\addtocounter{theorem}{-1}
\begin{theorem}[Multi-view contrastive estimation of noisy dynamics]
\label{thm:appendix_main_theorem}
Consider the latent Markov process on \(\mathbb R^d\)
\[
    \bm x_{t+1}=\rev{F}(\bm x_t)+\bm\varepsilon_t,
    \qquad
    \bm\varepsilon_t\overset{\mathrm{iid}}{\sim}\mathcal N(0,\Sigma),
    \qquad
    \Sigma\succ0,
\]
where \(\rev{F}:\mathbb R^d\to\mathbb R^d\) is a \(C^2\) diffeomorphism. Let \(\mathcal U\subseteq\mathbb R^d\) be a connected open region on which identification is claimed. Let \(q\) denote the marginal density of \(\bm x_t\), and assume that the clean density ratio
\[
    (\bm x,\bm x')\mapsto \frac{p(\bm x'\mid \bm x)}{q(\bm x')}
\]
is finite and continuous on \(\mathcal U\times\mathcal U\). Suppose that for each \(t\) we observe \(V\) independent noisy views
\[
    \bm y_t^a=g(\bm x_t)+\bm \xi_t^a,
    \qquad a=1,\dots,V,
\]
where \(g:\mathbb R^d\to\mathbb R^D\), \(D\ge d\), is a \(C^1\) bi-Lipschitz embedding with lower bi-Lipschitz constant \(m_g>0\). Assume that the centered noises \(\bm\xi_t^a-\bm\mu_\xi\) are independent, mean-zero, sub-Gaussian random vectors in \(\mathbb R^D\) with sub-Gaussian parameter \(\sigma_\xi^2\), uniformly in \(t\) and \(a\), where \(\bm\mu_\xi=\mathbb E[\bm\xi_t^a]\).

Consider the all-views DYSCO population model with encoder \(H_V:(\mathbb R^D)^V\to\mathbb R^d\), dynamics \(\rev{\widehat F_V}\), and potential \(\alpha_V\), with rollout score
\begin{equation}
    \Psi_V^{k_1,k_2}(\bm y_t,\bm y'_\tau)
    =
    -
    \left\|
        \rev{\widehat F_V^{\pm k_1}}(H_V(\bm y_t))
        -
        \rev{\widehat F_V^{\mp k_2}}(H_V(\bm y'_\tau))
    \right\|^2
    -
    \alpha_V(H_V(\bm y'_\tau)).
\end{equation}
Let
\begin{equation}
    \mathcal L_V[\Psi_V]
    =
    \sum_{k_1,k_2=0}^{\kappa}
    \lambda^{k_1,k_2}
    \mathcal L_V^{k_1,k_2}[\Psi_V].
\end{equation}
Assume the following.
\begin{enumerate}
    \item[\textnormal{(A1)}] The multi-horizon objective is jointly realizable: there exists a single tuple \((H_V,\rev{\widehat F_V},\alpha_V)\) attaining the infimum of every active constituent \(\mathcal L_V^{k_1,k_2}\). Moreover, the forward one-step constituent is active, \(\lambda^{1,0}>0\).

    \item[\textnormal{(A2)}] After fixing the affine gauge, consider any sequence of population global minimizers for which, along a subsequence, the learned all-views representation and dynamics have a stable \(C^1\) limit:
    \[
        H_V(\bm y_t^{(V)}) \longrightarrow r(\bm x_t)
        \qquad\text{in probability}
    \]
    for \(r\in C^1\), and \(\rev{\widehat F_V}\to \rev{\widehat F}\) in \(C^1_{\mathrm{loc}}\) on the learned support. The potential class is rich enough to represent the marginal term \(\log q\) on this support.
\end{enumerate}
Then in the limit $V, T \to \infty$ every such limiting minimizer identifies the latent state and the deterministic dynamics up to a common affine indeterminacy. That is, on \(\mathcal U\) there exist \(L\in\mathrm{GL}(d)\) and \(\bm b\in\mathbb R^d\) such that
\begin{equation}
    r(\bm x)=L\bm x+\bm b,
\end{equation}
and
\begin{equation}
    \rev{\widehat F}(\bm z)
    =
    L \rev{F}\left(L^{-1}(\bm z-\bm b)\right)+\bm b,
    \qquad
    \bm z\in r(\mathcal U).
\end{equation}
\end{theorem}

\begin{proof}
We prove the result in five steps.

\paragraph{Step 1: Independent views denoise the observation.}
For fixed \(t\), define the empirical all-views average
\[
    \bar{\bm y}_t^{(V)}
    =
    \frac1V\sum_{v=1}^V\bm y_t^v
    =
    g(\bm x_t)+\frac1V\sum_{v=1}^V\bm\xi_t^v,
\]
and set \(\tilde g(\bm x)=g(\bm x)+\bm\mu_\xi\). Then
\[
    \bar{\bm y}_t^{(V)}-\tilde g(\bm x_t)
    =
    \frac1V\sum_{v=1}^V(\bm\xi_t^v-\bm\mu_\xi).
\]
The right-hand side is sub-Gaussian with parameter \(\sigma_\xi^2/V\). Hence, for a universal constant \(C_0>0\), every \(u>0\), and every fixed \(t\),
\[
    p\left(
        \left\|
            \bar{\bm y}_t^{(V)}-\tilde g(\bm x_t)
        \right\|
        >
        C_0\sigma_\xi\sqrt{\frac{D+u}{V}}
    \right)
    \le e^{-u}.
\]
Equivalently, over any finite trajectory \(t=0,\dots,T\), a union bound gives, with probability at least \(1-\delta\),
\[
    \max_{0\le t\le T}
    \|\bar{\bm y}_t^{(V)}-\tilde g(\bm x_t)\|
    \le
    C_0\sigma_\xi
    \sqrt{\frac{D+\log((T+1)/\delta)}{V}}.
\]
In particular, for each fixed \(t\),
\[
    \mathbb E\|\bar{\bm y}_t^{(V)}-\tilde g(\bm x_t)\|
    \le
    C\sigma_\xi\sqrt{\frac{D}{V}}.
\]

\paragraph{Step 2: The all-views posterior concentrates on the clean latent state.}
Because \(g\) is bi-Lipschitz, \(\tilde g\) has the same lower bi-Lipschitz constant. Define the oracle inverse estimator
\[
    \hat{\bm x}_t^{(V)}
    \in
    \operatorname*{argmin}_{\bm x\in\mathbb R^d}
    \|\bar{\bm y}_t^{(V)}-\tilde g(\bm x)\|.
\]
The same projection argument as in the noiseless inverse problem gives
\[
    \|\hat{\bm x}_t^{(V)}-\bm x_t\|
    \le
    \frac{2}{m_g}
    \|\bar{\bm y}_t^{(V)}-\tilde g(\bm x_t)\|.
\]
Therefore
\[
    \mathbb E\|\hat{\bm x}_t^{(V)}-\bm x_t\|
    \le
    C\frac{\sigma_\xi}{m_g}\sqrt{\frac{D}{V}}
    \longrightarrow 0.
\]
where the \(L^1\) convergence follows directly from the sub-Gaussian bound.

Let \(\Pi_V(\cdot\mid \bm y_t^{(V)})\) be the posterior law of \(\bm x_t\) given all views at time \(t\). Since \(\hat{\bm x}_t^{(V)}\) is measurable with respect to \(\bm y_t^{(V)}\), the tower property yields
\[
    \mathbb E\left[
        \mathbb E[
            \|\hat{\bm x}_t^{(V)}-\bm x_t\|
            \mid
            \bm y_t^{(V)}
        ]
    \right]
    =
    \mathbb E\|\hat{\bm x}_t^{(V)}-\bm x_t\|
    \to0.
\]
Thus the conditional expectation converges to zero in probability. Conditional Markov's inequality gives, for every \(\varepsilon>0\),
\[
    \Pi_V\left(
        \{\bm x:\|\bm x-\hat{\bm x}_t^{(V)}\|>\varepsilon\}
        \mid
        \bm y_t^{(V)}
    \right)
    \le
    \frac{
        \mathbb E[
            \|\bm x_t-\hat{\bm x}_t^{(V)}\|
            \mid
            \bm y_t^{(V)}
        ]
    }{\varepsilon}
    \to0
\]
in probability. Since also \(\hat{\bm x}_t^{(V)}\to\bm x_t\) in probability, we conclude
\[
    \Pi_V(\cdot\mid \bm y_t^{(V)})
    \Rightarrow
    \delta_{\bm x_t}
    \qquad
    \text{in probability}.
\]

\paragraph{Step 3: The multi-horizon problem reduces to the active one-step problem.}
Let $\mathcal A$ be the set of active rollouts \(\mathcal A=\{(k_1,k_2):\lambda^{k_1,k_2}>0\}\). For each active constituent define
\[
    m^{k_1,k_2}
    =
    \rev{\inf_{H_V,\widehat F_V,\alpha_V}}
    \mathcal L_V^{k_1,k_2}(H_V,\rev{\widehat F_V},\alpha_V).
\]
By joint realizability, a single tuple attains all these infima. Since every active term is bounded below by its own infimum,
\[
    \mathcal L_V(H_V,\rev{\widehat F_V},\alpha_V)
    \ge
    \sum_{(k_1,k_2)\in\mathcal A}
    \lambda^{k_1,k_2}m^{k_1,k_2},
\]
and equality is attained by the jointly realizable tuple. Hence every global minimizer of the full objective minimizes each active constituent individually. In particular, since the forward one-step constituent is active, the identifiability analysis may be reduced to \(\mathcal L_V^{1,0}\).

\paragraph{Step 4: The Bayes-optimal all-views score converges to the clean latent score.}
Write \(\bm y=\bm y_t^{(V)}\) and \(\bm y'=\bm y_{t+1}^{(V)}\). The Bayes-optimal one-step contrastive score is, up to an anchor-only term \cite{wang2020understanding},
\[
    \psi_V^\star(\bm y,\bm y')
    =
    \log\frac{p_V(\bm y'\mid \bm y)}{q_V(\bm y')}
    +
    c_V(\bm y),
\]
where \(q_V\) is the marginal density of the negative all-views sample. Let
\[
    \Pi_V(d\bm x\mid \bm y)
    =
    p(\bm x_t\in d\bm x\mid \bm y_t^{(V)}=\bm y)
\]
and let \(\Pi_V^+(d\bm x'\mid \bm y')\) be the posterior law of \(\bm x_{t+1}\) under the marginal \(q\). Bayes' rule gives
\[
    \frac{p_V(\bm y'\mid \bm y)}{q_V(\bm y')}
    =
    \iint
        \frac{p(\bm x'\mid \bm x)}{q(\bm x')}
        \Pi_V(d\bm x\mid \bm y)
        \Pi_V^+(d\bm x'\mid \bm y').
\]
By Step 2, both posteriors concentrate:
\[
    \Pi_V(d\bm x\mid \bm y_t^{(V)})\Rightarrow\delta_{\bm x_t}(d\bm x),
    \qquad
    \Pi_V^+(d\bm x'\mid \bm y_{t+1}^{(V)})\Rightarrow\delta_{\bm x_{t+1}}(d\bm x').
\]
Since the clean density ratio is continuous and finite on the relevant support,
\[
    \psi_V^\star(\bm y_t^{(V)},\bm y_{t+1}^{(V)})
    \longrightarrow
    \log p(\bm x_{t+1}\mid \bm x_t)
    -
    \log q(\bm x_{t+1})
    +
    c(\bm x_t).
\]

\paragraph{Step 5: Score matching induces affine recovery.}
By \(\textnormal{(A1)}\), every global minimizer of the full objective realizes the Bayes-optimal active one-step score. By \(\textnormal{(A2)}\), along any stable limiting subsequence the model score therefore satisfies the clean score identity. At this point, we have reduced the problem to the clean observation case already tackled in \cite{laiz2025}. The remaining of the proof argument can thus proceed analogously as in \cite{laiz2025}. We substitute the expression \ref{eq:model_definition_full_views} for our full-view model for $\psi_V^\ast$ and the potential $\alpha'$ absorbs the marginal term \(\log q(\bm x')\), while anchor-only terms are collected in \(c'(\bm x)\). Using the Gaussian transition density and absorbing constant factors into the positive definite matrix \(\Lambda\), we obtain, for all \(\bm x,\bm x'\in\mathcal U\),
\begin{equation}
    -
    \left\|
        \rev{\widehat F}(r(\bm x))-r(\bm x')
    \right\|_2^2
    =
    -
    (\rev{F}(\bm x)-\bm x')^\top
    \Lambda
    (\rev{F}(\bm x)-\bm x')
    +
    c'(\bm x).
\label{eq:appendix_score_matching}
\end{equation}
Taking the mixed derivative with respect to \(\bm x\) and \(\bm x'\) yields
\[
    J_r(\bm x')^\top
    J_{\rev{\widehat F}}(r(\bm x))
    J_r(\bm x)
    =
    \Lambda \rev{J_F}(\bm x),
\]
up to the same harmless positive scalar absorbed into \(\Lambda\). The right-hand side is invertible because \(\Lambda\succ0\) and \(\rev{F}\) is a diffeomorphism. Hence \(J_r(\bm x')\) is full rank. Moreover, for any two points \(\bm x'_1,\bm x'_2\), fixing \(\bm x\) and comparing the two identities gives
\[
    J_r(\bm x'_1)^\top
    J_{\rev{\widehat F}}(r(\bm x))
    J_r(\bm x)
    =
    J_r(\bm x'_2)^\top
    J_{\rev{\widehat F}}(r(\bm x))
    J_r(\bm x).
\]
The same identity also implies that the middle-right factor is invertible, so \(J_r(\bm x'_1)=J_r(\bm x'_2)\). Thus \(J_r\) is constant on the connected region \(\mathcal U\), and \(r\) is affine:
\[
    r(\bm x)=L\bm x+\bm b,
    \qquad L\in\mathrm{GL}(d).
\]

It remains to identify the dynamics. Multiplying \eqref{eq:appendix_score_matching} by \(-1\), both sides are strictly convex in \(\bm x'\). The right-hand side has unique minimizer \(\bm x'=\rev{F}(\bm x)\). Since \(r\) is affine and invertible, the left-hand side has unique minimizer satisfying \(r(\bm x')=\rev{\widehat F}(r(\bm x))\). The identity holds pointwise in \(\bm x'\), so the minimizers coincide:
\[
    \rev{\widehat F}(r(\bm x))=r(\rev{F}(\bm x)).
\]
Writing \(r(\bm x)=L\bm x+\bm b\), we obtain
\[
    \rev{\widehat F}(\bm z)
    =
    L \rev{F}\left(L^{-1}(\bm z-\bm b)\right)+\bm b,
    \qquad
    \bm z\in r(\mathcal U),
\]
which concludes the proof.
\end{proof}

\end{revisiondiff}

\begin{remark}[Amortized Denoising]
Theorem \ref{thm:identifiability_theorem} offers theoretical guarantees in the ideal setting of the infinite views limit $V \to \infty$ and for an encoder with simultaneous access to the complete set of views. In practice, our model \ref{eq:model_definition} departs from these idealized assumptions and tackles the denoising problem in an amortized setting with a relative small amount of views. We take the successful system identification achieved in our experiments as evidence that such technique works well in practice.
\end{remark}

\vspace{5pt}

\begin{remark}[Role of multi-step rollouts]
The theorem allows the training objective to average over arbitrary rollout horizons \((k_1,k_2)\). The proof uses only the one-step constituent \((k_1,k_2)=(1,0)\), because under joint realizability any global minimizer of the full objective also minimizes each active constituent term individually. Thus the additional horizon terms do not weaken the identifiability conclusion but are not needed in the proof. However, in practice, for gradient-based optimizations, these additional terms can impose useful additional constraints that regularize the learning problem. For example, longer rollouts can help in the identification of the deterministic component in the presence of significant latent noise, at the cost of more expensive training iterations.
\end{remark}

\section{Dynamical Systems Definition}
\label{sec:appendix_dyn_sys}

In this section we provide the continuous-time formulation for the dynamical systems evaluated in this work, alongside the model used for the external forcing $\bm u_t$ if present.

\textbf{Duffing} The Duffing oscillator is a non-linear second-order differential equation with external forcing that can equivalently be represented as a system of two first-order ordinary differential equations of the form:
\begin{align}
\begin{split}
    \dot x_t &= v_t\\
    \dot v_t &= -\delta v_t -\alpha x_t -\beta x_t^3 + \gamma \cos (\omega t )
\end{split}
\label{eq:appendix_duffing}
\end{align}

In our experiment we fixed $\alpha = -1$, $\beta = 1$, $\delta = 0.3$, $\gamma = 0.3$ and $\omega = 1.2$. Note that in this case $\bm u_t = \cos ( \omega t )$ and $C = \left[0, \gamma \right]^\top$. We fixed $\bm \varepsilon_t \sim \mathcal{N}(0, \sigma_\epsilon \mathrm{Id})$, $\sigma_\epsilon = 0.05$ and used as integration time-step $\Delta t = 0.05$ with integration horizon $\kappa = 8$.

\textbf{Lorenz} The Lorenz chaotic attractor is a system of three first-order ordinary differential equations and it is a classic system that exhibits chaotic behavior. It is defined as:
\begin{align}
    \begin{split}
        \dot x_t &= \sigma ( y_t - x_t)\\
        \dot y_t &= x_t ( \rho - z_t) - y_t\\
        \dot z_t &= x_t y_t - \beta z_t
    \end{split}
\label{eq:appendix_lorenz}
\end{align}
In our experiment we fixed $\rho=28$, $\beta = 8/3$ and $\sigma = 10$. The latent noise was $\bm \varepsilon_t \sim \mathcal{N}(0, \sigma_\epsilon \mathrm{Id})$, $\sigma_\epsilon = 0.5$ and used as integration time-step $\Delta t = 0.01$ with integration horizon $\kappa = 8$.

\textbf{FitzHugh-Nagumo} The FitzHugh-Nagumo dynamical system is a relaxation oscillator and represents a prototype of an excitable system (e.g. a neuron). It is a forced dynamical system and is defined as:
\begin{align}
    \begin{split}
        \dot v_t &= v_t - \frac{v_t^3}{3} - w_t + I_\mathrm{ext}\\
        \tau \dot w_t &= v_t + a - b w_t
    \end{split}
\label{eq:appendix_fitzhug_nagumo}
\end{align}
In our experiment we fixed $a = 0.7$, $b = 0.8$, $\tau = 12.5$ and we model the external input current $I_\mathrm{ext}$ as a square wave with random amplitudes $| I_\mathrm{ext} | \sim U_{ \left[0, 1 \right]}$, period $T \sim U_{\left[ 8, 10 \right]}$ and $50\%$ duty-cycle. The latent noise was $\bm \varepsilon_t \sim \mathcal{N}(0, \sigma_\epsilon \mathrm{Id})$, $\sigma_\epsilon = 0.02$ and we used as integration time-step $\Delta t = 0.05$ with integration horizon $\kappa = 8$. Note that the Van-der-Pol oscillator is a special case of the FitzHugh–Nagumo model, with $a = b = 0$ and $C = \left[1, 0 \right]^\top$.

\textbf{Winner Take All} The winner-take-all dynamical systems are a class of system defined by the property of state self-excitation paired with reciprocal inhibition. We consider the following particular form the system:
\begin{align}
    \bm{\dot z}_t = \bm z_t \left( \zeta - \bm z_t^2 \right) - A \bm z_t + C \bm u_t,
\label{eq:appendix_winner_take_all}
\end{align}
where $\zeta = 1$ and $A$ is the interaction matrix which we take to be a matrix of all ones with a zero diagonal. In our experiments the state $\bm z_t$ had dimensionality $\bm z_t \in \mathbb{R}^3$. The external forcing $\bm u_t \in \mathbb{R}^d$ was again a random square wave with amplitude $| I | \sim U_{[0.1, 0.2]}$, period $T \sim U_{[10, 20]}$, $50\%$ duty-cycle and only one random component active per-cycle. We fixed $C = \mathrm{Id}$. The latent noise was $\bm \varepsilon_t \sim \mathcal{N}(0, \sigma_\epsilon \mathrm{Id})$, $\sigma_\epsilon = 0.01$ and we used as integration time-step $\Delta t = 0.05$ with integration horizon $\kappa = 8$. This particular system is often used to model competitive dynamics (\emph{e.g.} forced-choice decision-making in neuroscience).

\textbf{Double-Well} The double-well is a forced bistable system defined by the equation:
\begin{align}
    \begin{split}
        \dot x_t &= -\gamma \frac{\partial V}{\partial x} + \kappa \frac{\partial V}{\partial y}\\
        \dot y_t &=  -\gamma \frac{\partial V}{\partial x} - \kappa \frac{\partial V}{\partial y} + u_t\\[10pt]
        V &= \frac{x_t^4}{4} - \frac{x_t^2}{2} + \frac{y_t^2}{2}
    \end{split}
\label{eq:appendix_double_well}
\end{align}
where $V$ is a potential function and we fixed $\gamma = 0.7$ and $\kappa = 1.2$. As forcing function $u_t$ we used a square wave with amplitude $ I  \sim U_{[-2, +2]}$, period $T \sim U_{[8, 20]}$ and $50\%$ duty-cycle. For this system the control matrix is $C = [0, 1]^\top$. The latent noise was $\bm \varepsilon_t \sim \mathcal{N}(0, \sigma_\epsilon \mathrm{Id})$, $\sigma_\epsilon = 0.05$ and we used as integration time-step $\Delta t = 0.02$ with integration horizon $\kappa = 8$.

\textbf{Stuart-Landau} The Stuart-Landau system is a non-linear oscillator defined via the equation:
\begin{align}
    \begin{split}
        \dot x_t &= \mu x - \omega y - x \left(x^2 + y^2 \right)\\
        \dot y_t &= \omega x + \mu y - y \left( x^2 + y^2 \right)
    \end{split}
\label{eq:appendix_stuart_landau}
\end{align}
In our experiment we fixed $\mu = \omega = 1$. The latent noise was $\bm \varepsilon_t \sim \mathcal{N}(0, \sigma_\epsilon \mathrm{Id})$, $\sigma_\epsilon = 0.01$ and we used as integration time-step $\Delta t = 0.02$ with integration horizon $\kappa = 8$.

\textbf{Heteroclinic} An heteroclinic orbit is defined as a path in phase space that connects two saddle points. By heteroclinic system we identify here a metastable system defined as:
\begin{align}
    \begin{split}
        \dot x &= x \left( 1 - x - \alpha y - \beta z\right)  \\
        \dot y &= y \left( 1 - y - \alpha z - \beta x\right) \\
        \dot z &= z \left( 1 - z - \alpha x - \beta y\right)
    \end{split}
\label{eq:appendix_heteroclinic}
\end{align}

where we fixed $\alpha = 2$ and $\beta = \frac{1}{2}$. The latent noise was $\bm \varepsilon_t \sim \mathcal{N}(0, \sigma_\epsilon \mathrm{Id})$, $\sigma_\epsilon = 0.01$ and we used as integration time-step $\Delta t = 0.02$ with integration horizon $\kappa = 8$.

\section{Symbolic Recovery from Affine Orbits}
\label{sec:appendix_symbolic_recovery_affine_orbits}
We offer here a proof-of-concept application for the symbolic recovery of a latent dynamical system via identification of the sparsest representative of the system's coefficient representation along the affine group orbit.

We illustrate this on the Lorenz system in a noiseless setting as a proof of concept; a systematic study of symbolic recovery under noise, including the development of robust gauge-aware sparse regression, is an interesting direction for future work. In a noiseless configuration ($\sigma = 0$), our system achieves very accurate dynamics reconstruction ($\mathrm{dyn}R^2 = 99.11\%$) and we solved the optimization problem \ref{eq:sparsity_maximisation} via an $L_1$ proxy for sparsity with iterative pruning of small coefficients $| \theta_{kl}|\le 0.1$. We solved $200$ instances of the problem with different random initial conditions for both $L$ and $\bm b$. Here we report the best matching solution found.

\begin{table}[h]
\centering
\small
\setlength{\tabcolsep}{3pt}
\begin{tabularx}{\textwidth}{
>{\centering\arraybackslash}p{0.36\textwidth}
!{}
>{\centering\arraybackslash}p{0.60\textwidth}
}
\multicolumn{1}{c}{\textbf{Ground Truth System}}
&
\multicolumn{1}{c}{\textbf{Symbolic Recovery}}
\\
\midrule\\[-5pt]
$\begin{aligned}[t]
\dot x_t &= -10x_t + 10y_t\\
\dot y_t &= 28.0x_t - 1.0x_tz_t - 1.0y_t\\
\dot z_t &= 1.0x_ty_t - 2.7z_t
\end{aligned}$
&
$\begin{aligned}[t]
\dot x_t &= -10.0x_t + 2.9y_t {+0.13x_ty_t}\\
\dot y_t &= {-0.12x_ty_t} -2.1x_tz_t + 28.0x_t {-0.2y_tz_t} + 5.8y_t\\
\dot z_t &= {-0.8x_t^2} + 0.9x_ty_t {+0.82x_t} - 4.3z_t
\end{aligned}$
\end{tabularx}
\end{table}

While the overall structure is not far from the ground-truth, the equations contain several small spurious terms with coefficients below $1$, indicating that exact term recovery is not yet achieved and requires more advanced methods than simple $L_1$ thresholding within the gauge. Experimentally, we observed how minor imperfections in the flow field can hinder the symbolic regressor problem, which implicitly relies on exact term cancellations. Future work should experiment with more advanced frameworks, such as a paired teacher-student setup where a student network might learn to deviate from the exact affine group orbit in pursuit of significant sparsity gains. Additionally, more faithful approximations to the sparsity optimization problem \ref{eq:sparsity_maximisation}, \emph{e.g.} better $L_0$ approximations \cite{louizosLearningSparseNeural2018a, yangDeepHoyerLearningSparser2020, kolbDifferentiableSparsity$D$Gating2025}, might also be used to improve upon these initial attempts.

\section{Ablation Study on Integration Horizon}
\label{sec:appendix_ablation_integration_horizon}

\begin{figure}[tb]
    \centering
    \includegraphics[width=\textwidth]{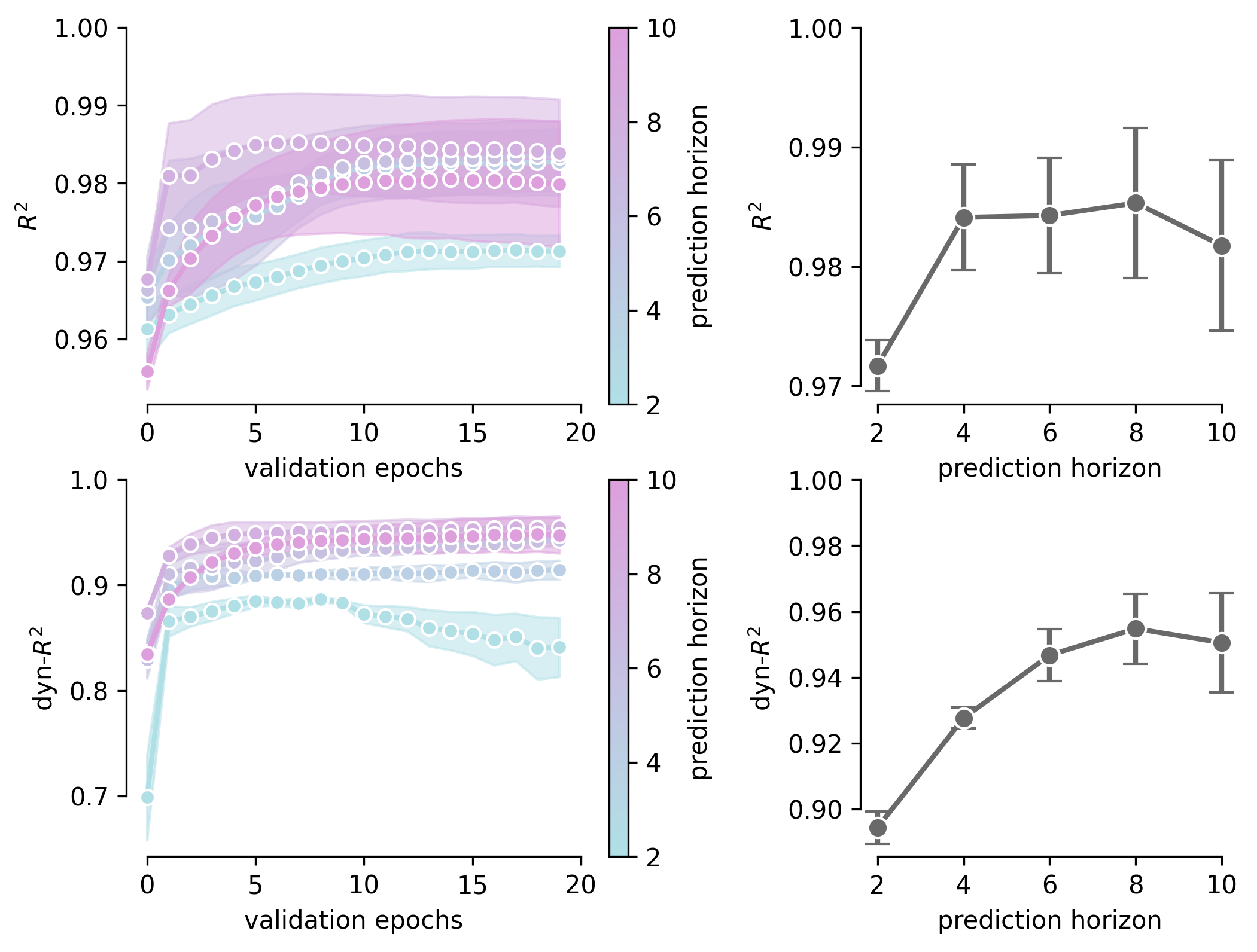}
    \caption{\textbf{Ablation on time horizon $\kappa$}. Ablation result for varying time integration horizons $\kappa = 2, 4, 6, 8, 10$ on the Lorenz system with Gaussian noise of intensity $\sigma = 0.3$. We trained for $200$ epochs and evaluated the model on the validation set every $10$ epochs (left column). We report the maximum achieved score on the right column for both the $R^2$ metric (top row) and $\mathrm{dyn}R^2$ metric (bottom row). Each marker is the average over 3 runs $\pm$ standard error of the mean.}
    \label{fig:ablation_time_horizon}
\end{figure}

One of the flexible components of the \textsc{DYSCO} algorithm \ref{eq:model_definition} is the maximal integration time-horizon $\kappa$. This parameter controls how far along the flow field the latent states are evolved before being compared via the similarity score. We argued in our second remark in Appendix section \ref{sec:appendix_theorem_proofs} how longer integration horizons might be beneficial in practice, despite Theorem \ref{thm:identifiability_theorem} claiming identification even for a single step of forward time-integration. Here we empirically investigate the impact of this hyperparameter on the model performances.

We consider the Lorenz system under Gaussian observation noise $\xi_t^a \sim \mathcal N (0, \sigma^2 \mathrm{Id})$ with intensity $\sigma = 0.3$. We trained our standard architecture ($4$-layer MLP encoder, polynomial functional basis, see Methods for details) for $200$ epochs and measured both $R^2$ and $\mathrm {dyn} R^2$ metrics on the held-out validation set every $10$ epochs. We explored the following time integration horizons $\kappa = 2, 4, 6, 8, 10$ and report our results in Figure \ref{fig:ablation_time_horizon}. We note that indeed longer horizons result in improved quality for both latent trajectories and flow fields, with a saturating (or peaking) effect after $\kappa = 8$. In practice, however, longer integration horizons come with a direct computational cost: on our hardware (NVIDIA RTX A4000) a full training run with $\kappa = 2$ lasted $\sim50$min, while an identical configuration for $\kappa = 10$ took $\sim2$h $20$min; with the majority of the overhead coming from the \textsc{odeint} rollout. Furthermore, we noted how longer rollouts ($\kappa = 10$) tended to suffer from training instability, probably due to the inherent difficulty of bootstrapping flow-fields from random initializations that are stable for longer horizons. For these reasons, we fixed $\kappa = 8$ for our main experiments.


\begin{revision}
\section{Controlled Baseline Protocol and Training Budgets}
\label{sec:appendix_baseline_protocol}
The controlled comparisons use fixed noisy Lorenz and Stuart--Landau observations with one, two and four views; Figure~\ref{fig:baseline_comparisons} displays one and four views. DYSCO uses view correspondences, while the baselines receive the same observations as additional trajectories. Gaussian observation noise has $\sigma=1.0$ for Lorenz and $\sigma=0.2$ for Stuart--Landau. The Poisson condition uses the same raw observations across methods at each view count.

\paragraph{Baseline Methods.}
DCL uses a differentiable switching-linear model with $200$ modes and $20,000$ negative samples. SINDy-AE uses a degree-three polynomial library, encoder widths $64/32$, and its reconstruction/dynamics objectives. RP-GSSM uses a Gaussian latent state-space model with a nonlinear recognition network, sequence length $100$, and stabilized transition parameters. Architectures, derivative preprocessing, batch units and optimization differ across methods. The comparison controls raw observational information, with additional alignment metadata supplied to DYSCO.

\paragraph{Architectures and configured schedules.}
The coefficient network generates one global coefficient matrix from a learned source vector. In these controlled comparisons, Lorenz DYSCO uses source dimension $32$, hidden widths $128/128$, a quadratic library and learning rate $10^{-4}$; Stuart--Landau uses source dimension $32$, widths $64/64$, a cubic library and learning rate $3\times10^{-4}$. Lorenz configurations use $200$ or $500$ epochs, and Stuart--Landau configurations use $200$ epochs. The neural model's active polynomial term uses source dimension $8$ and one width-$32$ hidden layer. Table~\ref{tab:training-budgets} summarizes the configured schedules.

\begin{table}[htbp]\centering\color{cameraready}\small
\begin{tabularx}{\linewidth}{lXrr}
\toprule
Method & Configured schedule & Batch size & Learning rate \\
\midrule
DYSCO & 200 or 500 epochs & 4096 & $10^{-4}$ or $3\!\times\!10^{-4}$ \\
DCL & 30,000 iterations & 2048 & $3\!\times\!10^{-4}$ \\
SINDy-AE & 10,001 epochs + 1,001 refinement & 8000 & $10^{-3}$ \\
RP-GSSM & 10,000 iterations & 32 sequences & $10^{-3}$ \\
\bottomrule
\end{tabularx}
\caption{Configured baseline training budgets from 149 run configurations. Batch units and per-step objectives vary across methods; RP-GSSM uses sequences of length 100.}
\label{tab:training-budgets}
\end{table}

\section{Motor-Cortical Data and Analysis Protocol}
\label{sec:appendix_neural}
\paragraph{Data and fitting.}
We use MC\_Maze motor-cortical spike counts from the Neural Latents Benchmark \citep{pei2021neural}, in 20-ms bins over $[-250,450)$ ms relative to movement onset. Conditions are defined by trial type and version, retaining groups with sufficient repetitions for four-view training. The split contains 1,721 training and 574 disjoint validation trials across 108 matched conditions, with 137 observed neurons. We learn six latent coordinates and a degree-two polynomial field with a linear observed-input term. The go cue is rendered as a Gaussian pulse with width parameter 40 ms. The input transform combines Freeman--Tukey variance stabilization, a learned five-bin convolution and batch normalization. We trained for 500 epochs and used the saved validation-loss-selected checkpoints for analysis. Behavioral variables are used for evaluation.

\paragraph{Rotational analysis.}
We adapt the jPCA analysis of \citet{churchland2012neural} over $[-50,150)$ ms. After subtracting the cross-condition mean separately at each time, we fit the six-PC transform and skew-symmetric linear generator on training condition means. The dominant plane corresponds to the largest imaginary eigenfrequency. We orient its axes using training data to maximize initial spread along jPC1 and make net rotation counterclockwise. Plotted curves in Figure~\ref{fig:neural_analysis} use validation repetitions, colored by their displayed starting jPC1 coordinate. Gaussian smoothing ($\sigma=20$ ms, truncated at three standard deviations) is applied to the full projected time series before cropping, solely for display. Fitted axes and quantitative results use unsmoothed arrays.

Table~\ref{tab:neural-rotation-context} reports the numerical value of the neural analysis for the skew/full ratios and the measured rotational frequencies. DYSCO uncertainties summarize three independent initializations, while neural PCA is a deterministic same-data control.

\begin{table}[H]
\centering\color{cameraready}\small
\setlength{\tabcolsep}{4pt}
\resizebox{\linewidth}{!}{\begin{tabular}{lccc}
\toprule
Quantity & DYSCO & Neural PCA & Churchland et al. \\
\midrule
Six-dimensional skew/full ratio & $0.848 \pm 0.017$ & 0.622 & $\approx 0.820 \pm 0.076$ ($n=9$) \\
Dominant-plane skew/full ratio & $0.970 \pm 0.017$ & 0.820 & $\approx 0.894 \pm 0.053$ ($n=7$) \\
Angular frequency (rad/s) & $6.057 \pm 0.993$ & 7.419 & A: 8.8 / 9.8; B: 12.2 / 13.5 \\
\bottomrule\end{tabular}}
\caption{\textbf{Rotational analysis of MC\_Maze.} DYSCO: mean $\pm$ sample SD over three seeds; neural PCA: a deterministic same-data control. Values for the neural data are extracted from \citet{churchland2012neural}, Fig.~6b: mean $\pm$ sample SD across nine six-dimensional and seven best-plane supplied points, summarizing variation among the supplied points. Angular frequencies are reported for monkeys A/B, fast/slow, while MC\_Maze here is monkey J.}
\label{tab:neural-rotation-context}
\end{table}

\paragraph{Velocity probes and absolute field accuracy.}
We employed ridge probes to predict the signed two-dimensional hand velocity at $+100$ ms. We selected regularization strength by sample-wise five-fold cross-validation on the training set, using nine logarithmically spaced values from $10^{-4}$ to $1$ and unscaled features. The final evaluation uses disjoint validation trials from the same conditions. Condition-average probes are fitted separately to training condition means and evaluated on validation condition means. Table~\ref{tab:neural_decoding_features} reports the results for linear and quadratic features of the state and learned field. The quadratic expansion of state and field includes field squares and state--field interactions, increasing the feature count from $27$ to $90$. Since the field is a function of the state and observed input, these results characterize the usefulness of its features under each readout class. The learned field's separate full-window held-out derivative $R^2$ is $0.138\pm0.005$, with cosine agreement $0.355\pm0.002$.

\begin{table}[htb]
\centering\color{cameraready}\small
\setlength{\tabcolsep}{5pt}
\begin{tabular}{lrrr}
\toprule
Readout features & Features & Single-trial $R^2$ & Condition-mean $R^2$ \\
\midrule
State & 6 & $0.267\pm0.014$ & $0.328\pm0.017$ \\
State and field & 12 & $0.402\pm0.015$ & $0.524\pm0.014$ \\
Quadratic state & 27 & $0.545\pm0.015$ & $0.705\pm0.016$ \\
Quadratic state and linear field & 33 & $0.545\pm0.015$ & $0.705\pm0.014$ \\
Quadratic state and field & 90 & $0.649\pm0.015$ & $0.848\pm0.015$ \\
\bottomrule\\
\end{tabular}
\caption{\textbf{Velocity decoding from state and field features.} Mean $\pm$ sample SD over seeds 42--44, with the same trials and $+100$-ms velocity targets in each comparison. Quadratic features include all linear, square and pairwise-product terms. The 33-feature readout appends the six field coordinates to the quadratic state library; the 90-feature readout expands all twelve state and field coordinates jointly. Ridge intercepts are fitted separately.}
\label{tab:neural_decoding_features}
\end{table}

\paragraph{Projected field and nearby flows.}
For the geometric illustration in Figure~\ref{fig:neural_analysis}a, we use seed 43, selected by training cross-validation of the state-plus-field velocity probe. The PCA projection is fitted to training condition means and applied to all 108 validation condition means. Dark arrows show $40$ ms times the field evaluated on individual validation trials with their observed go-cue inputs, averaged within condition and projected into the same PCs. To illustrate the field around the observed states, we generate nearby starting points around twelve evenly spaced condition indices at $-150$, $-10$, $+130$, $+270$ and $+410$ ms. Offsets combine $-0.5$, $0$ and $+0.5$ training-state standard deviations along PC1--3, preserving the other three coordinates at each anchor. Eight starting points are selected using projected separation and open display regions. We integrate the full six-dimensional field from these points for $80$ ms in both directions, with the go-cue input fixed at $u=0$, using fixed-step RK4 with a $2.5$-ms step. The grey curves and their arrows show these nearby flows, with all six coordinates evolving during integration.

\paragraph{Symbolic field in a rotational gauge.}
To better illustrate the potential application of DYSCO in analyzing the symbolic form of the fitted dynamics, we inspected the same frozen seed-43 model to see whether a clean rotational structure was evident in the dynamics itself. Given the gauge-dependence of the symbolic form (Appendix \ref{sec:appendix_symbolic_recovery_affine_orbits}), we seeked a particular gauge where the rotational form might become immediately evident: a natural choice is the rotation identified by the jPCA construction itself. In practise, we fixed an orthogonal basis $U\in\mathbb R^{6\times6}$, whose first two columns span the training-fitted jPCA plane, with the orientation described in the main text, while for the remaining columns we picked the principal directions of training-state variation in the orthogonal complement. With $\bm m$ the mean of all training states, we defined:
\begin{equation}
    \bm q=U^\top(\bm z-\bm m),\qquad
    \widehat f_q(\bm q,u)=U^\top\widehat f(\bm m+U\bm q,u).
\label{eq:neural_field_coordinates}
\end{equation}
The fixed origin and basis preserve the autonomous part of the field under a time-independent affine change of coordinates. The quadratic model can therefore be written as:
\begin{equation}
    \widehat f_q(\bm q,u)=\bm c+A\bm q+
    \sum_{1\leq i\leq j\leq6}\bm b_{ij}q_iq_j+\bm d\,u,
\label{eq:neural_field_polynomial}
\end{equation}
with $A$ identifying the linear part of the field in this particular gauge (see Figure~\ref{fig:neural_symbolic_analysis}a), while the remaining terms capture the nonlinear and input-dependent contributions. The linear part of the field is also reported in the main text in Eq.~\eqref{eq:neural_rotating_field}, where we rounded the coefficients to two significant figures for display.

This linear terms shows a clear rotational structure in its dominant $2\times2$ block ($q_1$ and $q_2$), with the antisymmetric part of this block being characterized by a rate $(A_{21}-A_{12})/2=5.13$ rad/s. This block describes the local linear contribution at the training mean, where the field has nonzero drift $\bm c$, which differs from the jPCA frequency in Table~\ref{tab:neural-rotation-context}.

The remainders in Eq.~\eqref{eq:neural_rotating_field} contain $c_i$, linear couplings to $q_3,\ldots,q_6$, all quadratic terms and the observed-input response $d_i u$. We then quantified the relative importance of these terms by examining their contribution to the overall field. Completly omitting them gives relative RMS errors of $73.6\%$ and $25.3\%$ in the first two field components at held-out states with $u=0$. Omitting every dependence on $q_3,\ldots,q_6$ gives $26.7\%$ relative RMS error in the two-component field, indicating a measurable but not dominant contribution of the remaining four coordinates to the rotational dynamics.

\paragraph{Progressive term removal.}
To quantify the cost of simplifying the equations, we rank the $168$ scalar coefficients of the state library (one constant, six linear and 21 quadratic terms in each of six equations) by their RMS contribution on training states. For coefficient $\theta_{ij}$ multiplying monomial $\phi_j$, the relative contribution is:
\begin{equation}
    s_{ij}=\frac{|\theta_{ij}|\sqrt{\langle\phi_j(\bm q)^2\rangle_{\mathrm{train}}}}
    {\max_k |\theta_{ik}|\sqrt{\langle\phi_k(\bm q)^2\rangle_{\mathrm{train}}}}.
\label{eq:neural_term_contribution}
\end{equation}
We remove coefficients in ascending order of $s_{ij}$, one at a time, keeping all six input coefficients fixed and performing no refitting. For each of the 169 retained-count settings, Figure~\ref{fig:neural_symbolic_analysis}b reports the relative RMS difference from the full learned field at validation states with $u=0$. We evaluate both the complete field and its first two components (which by construction define the rotating plane):
\begin{equation}
    E_S(K)=\frac{\left\|P_S\bigl(\widehat f_q^{(K)}-\widehat f_q\bigr)\right\|_{F,\mathrm{val}}}
    {\left\|P_S\widehat f_q\right\|_{F,\mathrm{val}}},
    \qquad S\in\{\{1,\ldots,6\},\{1,2\}\},
\label{eq:neural_term_removal_error}
\end{equation}
where $K$ is the number of retained state-library coefficients and $P_S$ selects the indicated components. For example, retaining coefficients above $5\%$ of the largest training RMS contribution in each equation leaves 112 coefficients and gives $10.8\%$ plane-field error; a $10\%$ threshold leaves 80 coefficients and gives $15.8\%$ error. These errors measure the approximation of the fitted polynomial field, with all retained contributions summed before evaluation. The resulting curves indicate that the learned field can be substantially simplified at a modest cost in accuracy, with the first two rotating components being somewhat more robust to simplification than the full six-dimensional field.

\begin{figure}[H]
    \centering\color{cameraready}
    \includegraphics[width=\linewidth]{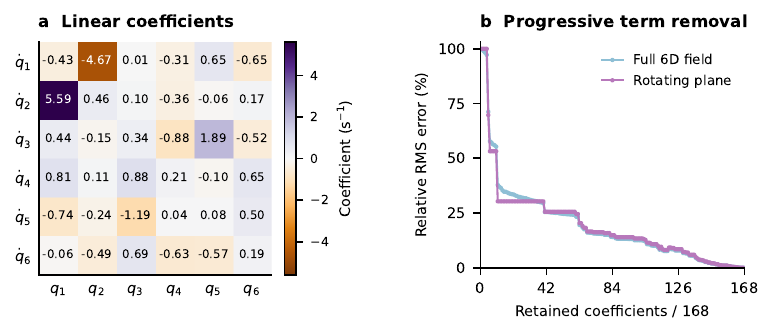}
    \caption{\textbf{Symbolic structure of the learned neural field.} Seed 43 in the fixed coordinates of Eq.~\eqref{eq:neural_field_coordinates}. (\textbf{a}) Full linear coefficient matrix $A$, with entries rounded for display. (\textbf{b}) Relative RMS approximation error at held-out states with $u=0$ as individual coefficients are removed in ascending order of their relative training RMS contribution. Each curve contains all 169 settings, from 168 retained state-library coefficients to zero. Cyan evaluates all six field components; orchid evaluates the two rotating components. Input coefficients remain fixed.}
    \label{fig:neural_symbolic_analysis}
\end{figure}

\section{Forward-Only Objective and Overcomplete-Library Diagnostics}
\label{sec:appendix_objective_library}
\paragraph{Forward-only objective.} We validated the theoretical prediction from Theorem \ref{thm:identifiability_theorem} that a forward-only loss is sufficient for identifying the underlying dynamics. We considered the four-view Gaussian Lorenz system and compared forward-only training (at horizon $k=8$) against the full DYSCO objective. We report results in Table~\ref{tab:forward-only} for three seed-matched $200$-epoch runs. Indeed, we find that forward-only training is sufficient for recovering the system; althought the additional terms of the full DYSCO objective improve field recovery accuracy in this experiment, indicating how their inclusion provides additional valuable training signal.

\begin{table}[H]\centering\color{cameraready}\small
\begin{tabular}{lrrr}
\toprule
Objective & $n$ & Latent $R^2$ (\%) & Field dyn$R^2$ (\%) \\
\midrule
Full objective & 3 & $98.0 \pm 1.2$ & $95.1 \pm 2.1$ \\
Forward only ($k=8$) & 3 & $96.0 \pm 0.8$ & $87.5 \pm 2.1$ \\
\bottomrule\vspace{5pt}
\end{tabular}
\caption{Four-view Gaussian Lorenz objective ablation. Mean $\pm$ SEM over seeds 42--44; the forward-only variant retains only the forward/encoded contrastive term. The extra objective terms improve field fit in this setting; the theorem uses an active forward one-step constituent.}
\label{tab:forward-only}
\end{table}

\paragraph{Overcomplete-library diagnostic.} As a preliminary assessment of DYSCO resilience against symbolic library misspecification, we measured the effects of an overcomplete-library by comparing a quadratic and cubic field library on the Lorenz system at a matched seed. A cubic library doubles the coefficient count for this system from $30$ to $60$, resulting in a more complex model to optimize. We train the two DYSCO configurations for $200$ epochs and report the results in Table~\ref{tab:library-degree}. We observe that the cubic library yields a lower field precision in this single-run comparison, although the quality degradation is gradual and not catastrophic.

\begin{table}[H]\centering\color{cameraready}\small
\begin{tabular}{lrrr}
\toprule
Library & Coefficients & Latent $R^2$ (\%) & Field dyn$R^2$ (\%) \\
\midrule
Quadratic & 30 & 99.15 & 97.22 \\
Cubic & 60 & 97.21 & 85.94 \\
\bottomrule\vspace{5pt}
\end{tabular}
\caption{Overcomplete-library diagnostic on four-view Gaussian Lorenz: single matched seed 42, trained for 200 epochs. Entries are the scores of the individual runs.}
\label{tab:library-degree}
\end{table}

\section{Autonomous Lorenz Climate}
\label{sec:appendix_climate}
To assess the long-term stability of the learned DYSCO dynamics, we evaluate our stored DYSCO checkpoints on the Lorenz system for additional 50-time-unit autonomous rollouts from 24 held-out initial states per model, using the already-computed affine alignment and a 10-time-unit burn-in for attractor statistics. Figure~\ref{fig:lorenz_climate}a,b shows representative long-term trajectories for the Lorenz system with Gaussian and Poisson noise, respectively. Figure~\ref{fig:lorenz_climate}c,d shows the corresponding power spectra and survival rates of the trajectories (i.e. time for which the trajectories remains stable). Table~\ref{tab:lorenz-climate} reports values for all seven available four-view checkpoints evaluations.

\begin{figure}[H]\centering\color{cameraready}
\includegraphics[width=\linewidth]{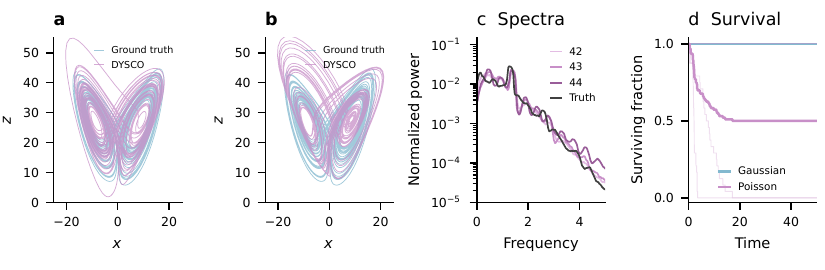}
\caption{\textbf{Autonomous Lorenz dynamics after four-view training.} (\textbf{a,b}) Ground-truth and affine-aligned learned trajectories for Gaussian seed 42 and selected successful Poisson seed 45, respectively; first cached initial condition, after a 10-time-unit burn-in. Chaotic trajectories can diverge in phase. (\textbf{c}) Gaussian power spectra for seeds 42--44 (orchid shades), averaged over surviving trajectories and coordinates with Hann-window/nine-bin smoothing. (\textbf{d}) Survival among 24 initial conditions over 50 time units: thick solid lines average three Gaussian seeds (42--44) and four Poisson seeds (42--45); thin transparent solid lines show individual seeds.}
\label{fig:lorenz_climate}
\end{figure}

\begin{table}[H]\centering\color{cameraready}\small
\begin{tabular}{lrrrrrr}
\toprule
Noise & Seed & dyn$R^2$ & Survival & Overlap & Spectrum $H$ & Switch rate \\
\midrule
Gauss & 42 & 0.972 & 100\% & 0.502 & 0.153 & 0.759 \\
Gauss & 43 & 0.971 & 100\% & 0.562 & 0.192 & 0.864 \\
Gauss & 44 & 0.909 & 100\% & 0.333 & 0.191 & 0.934 \\
Poisson & 42 & 0.939 & 0\% & 0.000 & --- & --- \\
Poisson & 43 & 0.518 & 100\% & 0.038 & 0.784 & 1.928 \\
Poisson & 44 & 0.934 & 0\% & 0.238 & --- & --- \\
Poisson & 45 & 0.965 & 100\% & 0.372 & 0.181 & 0.827 \\
\bottomrule\vspace{5pt}
\end{tabular}
\caption{Four-view Lorenz autonomous rollouts over 50 time units. All available four-view climate evaluations are shown; seed 45 is an additional Poisson run. Statistics discard the first 10 time units; overlap is $1-\mathrm{TV}$ (higher is better), $H$ is spectral Hellinger distance (lower is better). Dashes indicate unavailable spectral/switch statistics after early escape. Gaussian ground-truth switching rate is 0.579 per time unit. High local field accuracy does not guarantee autonomous stability.}
\label{tab:lorenz-climate}
\end{table}

For the case of Gaussian noise, we measured local dyn$R^2$ $0.951\pm0.036$ with survival rate hitting 100\% and attractor occupancy overlap of $0.466\pm0.119$ and spectral Hellinger distance of $0.179\pm0.022$ (mean $\pm$ SD). The lobe-switching rate is $0.852\pm0.088$ per time unit, compared with $0.579$ for the matched ground-truth cohort. Together this metrics indicate a good qualitative long-term behavior of the learned dynamics, although some appreciable statistical differences remain.

For the case of Poisson noise we observed how some seeds yield unstable long-term dynamics, resulting in less than 100\% survival. The surviving seeds yield local dyn$R^2$ $0.950\pm0.024$, attractor occupancy overlap of $0.205\pm0.167$, spectral Hellinger distance of $0.483\pm0.429$ and lobe-switching rate of $1.378\pm0.556$ per time unit, indicating a more significant deviation from the ground-truth dynamics. This preliminary analysis suggests that the Poisson noise case is more challenging for long-term stability, and that further work is needed to improve the robustness of the learned dynamics under this noise model.
\end{revision}

\clearpage
\section*{NeurIPS Paper Checklist}

\begin{enumerate}

\item {\bf Claims}
    \item[] Question: Do the main claims made in the abstract and introduction accurately reflect the paper's contributions and scope?
    \item[] Answer: \answerYes{}.
    \item[] Justification: The abstract and the introduction present our new algorithm for latent dynamical system identification, claiming both theoretical guarantees and experimental validation. The main theoretical result is presented in section \ref{sec:contrastive_learning_identifies}, while sections \ref{sec:experimental_setup} and \ref{sec:results} showcase our experimental results.

\item {\bf Limitations}
    \item[] Question: Does the paper discuss the limitations of the work performed by the authors?
    \item[] Answer: \answerYes{}
    \item[] Justification: We provide a dedicated subsection on the main limitations of our algorithm in our Discussion section \ref{sec:discussion}.

\item {\bf Theory assumptions and proofs}
    \item[] Question: For each theoretical result, does the paper provide the full set of assumptions and a complete (and correct) proof?
    \item[] Answer: \answerYes{}
    \item[] Justification: We provide a full proof for our Theorem \ref{thm:identifiability_theorem} in Appendix \ref{sec:appendix_theorem_proofs}.

    \item {\bf Experimental result reproducibility}
    \item[] Question: Does the paper fully disclose all the information needed to reproduce the main experimental results of the paper to the extent that it affects the main claims and/or conclusions of the paper (regardless of whether the code and data are provided or not)?
    \item[] Answer: \answerYes{}
    \item[] Justification: For all our experiments we provide a dedicated section describing the data generating process, model architecture and optimizer used (section \ref{sec:experimental_setup}). In Appendix \ref{sec:appendix_dyn_sys} we provide a full description of the dynamical systems considered in our experiments with corresponding parameters. Furthermore, we will fully open-source the code upon acceptance.

\item {\bf Open access to data and code}
    \item[] Question: Does the paper provide open access to the data and code, with sufficient instructions to faithfully reproduce the main experimental results, as described in supplemental material?
    \item[] Answer: \answerYes{}
    \item[] Justification: We provide the source code used in our experiments, together with configuration files and training scripts for ease-of-use.

\item {\bf Experimental setting/details}
    \item[] Question: Does the paper specify all the training and test details (e.g., data splits, hyperparameters, how they were chosen, type of optimizer) necessary to understand the results?
    \item[] Answer: \answerYes{}
    \item[] Justification: In our Section \ref{sec:experimental_setup} we describe our full data generating process, model architecture, dataset size, training details (optimizers, used hardware, approximate training time). Additionally Appendix \ref{sec:appendix_dyn_sys} collects all the parameters for each dynamical systems considered.

\item {\bf Experiment statistical significance}
    \item[] Question: Does the paper report error bars suitably and correctly defined or other appropriate information about the statistical significance of the experiments?
    \item[] Answer: \answerYes{}
    \item[] Justification: For both $R^2$ and $\mathrm{dyn}R^2$, we report means over three runs. Table~\ref{tab:r2_scores} reports standard deviations. The baseline comparison (Figure~\ref{fig:baseline_comparisons}) and ablations (Figures~\ref{fig:ablation_results} and~\ref{fig:ablation_time_horizon}) show standard errors of the mean.

\item {\bf Experiments compute resources}
    \item[] Question: For each experiment, does the paper provide sufficient information on the computer resources (type of compute workers, memory, time of execution) needed to reproduce the experiments?
    \item[] Answer: \answerYes{}
    \item[] Justification: In our \textbf{Training details} subsection of Section \ref{sec:experimental_setup} we provide details about the hardware used to run all our experiments and the approximate training time required for a typical experiment.
    
\item {\bf Code of ethics}
    \item[] Question: Does the research conducted in the paper conform, in every respect, with the NeurIPS Code of Ethics \url{https://neurips.cc/public/EthicsGuidelines}?
    \item[] Answer: \answerYes{}
    \item[] Justification: We conform with the NeurIPS Code of Ethics.

\item {\bf Broader impacts}
    \item[] Question: Does the paper discuss both potential positive societal impacts and negative societal impacts of the work performed?
    \item[] Answer: \answerNA{}
    \item[] Justification: Our main contribution is the introduction of a novel algorithm for latent dynamical system identification. Given the foundational nature of this research, we do not foresee potential negative societal impacts, while the potential benefits (neuroscience applications) are currently far downstream.

\item {\bf Safeguards}
    \item[] Question: Does the paper describe safeguards that have been put in place for responsible release of data or models that have a high risk for misuse (e.g., pre-trained language models, image generators, or scraped datasets)?
    \item[] Answer: \answerNA{}
    \item[] Justification: This paper poses no such risks.

\item {\bf Licenses for existing assets}
    \item[] Question: Are the creators or original owners of assets (e.g., code, data, models), used in the paper, properly credited and are the license and terms of use explicitly mentioned and properly respected?
    \item[] Answer: \answerYes{}
    \item[] Justification: In our experiments we use several existing open python libraries (e.g. \textsc{torchdiffeq}), which we properly cite.
    
\item {\bf New assets}
    \item[] Question: Are new assets introduced in the paper well documented and is the documentation provided alongside the assets?
    \item[] Answer: \answerNA{}
    \item[] Justification: The paper does not release new assets.

\item {\bf Crowdsourcing and research with human subjects}
    \item[] Question: For crowdsourcing experiments and research with human subjects, does the paper include the full text of instructions given to participants and screenshots, if applicable, as well as details about compensation (if any)? 
    \item[] Answer: \answerNA{}
    \item[] Justification: The paper does not involve crowdsourcing nor research with human subjects.

\item {\bf Institutional review board (IRB) approvals or equivalent for research with human subjects}
    \item[] Question: Does the paper describe potential risks incurred by study participants, whether such risks were disclosed to the subjects, and whether Institutional Review Board (IRB) approvals (or an equivalent approval/review based on the requirements of your country or institution) were obtained?
    \item[] Answer: \answerNA{}
    \item[] Justification: The paper does not involve crowdsourcing nor research with human subjects.

\item {\bf Declaration of LLM usage}
    \item[] Question: Does the paper describe the usage of LLMs if it is an important, original, or non-standard component of the core methods in this research? Note that if the LLM is used only for writing, editing, or formatting purposes and does \emph{not} impact the core methodology, scientific rigor, or originality of the research, declaration is not required.
    \item[] Answer: \answerNA{}
    \item[] Justification: The core method developed in this research does not involve LLMs as any important components.

\end{enumerate}

\end{document}